\documentclass{article}

    \usepackage[preprint]{neurips_2024}

\usepackage[utf8]{inputenc} 
\usepackage[T1]{fontenc}    
\usepackage{hyperref}       
\usepackage{url}            
\usepackage{booktabs}       
\usepackage{amsfonts}       
\usepackage{nicefrac}       
\usepackage{microtype}      
\usepackage{xcolor}         

\usepackage{algorithm}
\usepackage{algorithmic}
\usepackage{macros}
\usepackage{subfigure}

\usepackage{verbatim}

\usepackage{dirtytalk}

\usepackage[capitalize,noabbrev]{cleveref}

\usepackage{thmtools}
\declaretheorem[name=Theorem,refname={Theorem,Theorems},Refname={Theorem,Theorems}]{theorem}
\declaretheorem[name=Lemma,refname={Lemma,Lemmas},Refname={Lemma,Lemmas},sibling=theorem]{lemma}

\declaretheorem[name=Definition,refname={Definition,Definitions},Refname={Definition,Definitions},sibling=theorem]{definition}

\usepackage[textsize=tiny]{todonotes}

\title{Optimal Design for Adaptive In-Context Prompt Tuning in Large Language Models}

\author{%
  Subhojyoti Mukherjee\thanks{Work conducted during an internship at Amazon.} \\
  ECE Department\\
  UW-Madison\\
  Wisconsin, Madison \\
  \texttt{smukherjee27@wisc.edu} \\
  \and
  Anusha Lalitha \\
  AWS AI Labs\\
  Santa Clara\\
  USA \\
  \and
  Aniket Deshmukh \\
  AWS AI Labs\\
  Santa Clara\\
  USA \\
  \and
  Ge Liu \\
  AWS AI Labs\\
  Santa Clara\\
  USA \\
  \and
  Yifei Ma \\
  AWS AI Labs\\
  Santa Clara\\
  USA \\
  \and
  Branislav Kveton \\
  AWS AI Labs\\
  Santa Clara\\
  USA 
}

\begin{document}

\maketitle

\begin{abstract}
One emergent ability of large language models (\LLMs) is that query-specific examples can be included in the prompt at inference time. In this work, we use active learning for adaptive prompt design and call it \textbf{A}ctive \textbf{I}n-context \textbf{P}rompt \textbf{D}esign (\aicl). We design the \LLM\ prompt by adaptively choosing few-shot examples from a training set to optimize performance on a test set. The training examples are initially unlabeled and we obtain the label of the most informative ones, which maximally reduces uncertainty in the \LLM\ prediction. We propose two algorithms, \go\ and \sal, which differ in how the few-shot examples are chosen. We analyze these algorithms in linear models: first \go\ and then use its equivalence with \sal. We experiment with many different tasks in small, medium-sized, and large language models; and show that \go\ and \sal\ outperform other methods for choosing few-shot examples in the \LLM\ prompt at inference time.
\end{abstract}

\section{Introduction}
\label{sec:intro}
Large language models (\LLMs), such as Vicuna \citep{vicuna2023}, Falcon-40B  \citep{refinedweb}, and OpenLLaMA \citep{touvron2023llama} are applied in mainly two ways: fine-tuning and prompt tuning. In fine-tuning, the \LLM\ weights are adapted to a downstream task \citep{devlin2018bert}. Fine-tuning can easily incorporate domain knowledge that a pre-trained model may not possess and resembles classic inductive inference. Fine-tuned models often do not need carefully designed prompts, which makes them easier to deploy. The main drawback of fine-tuning is that it can be costly, because tens of thousands of training examples may be needed to fine-tune billions of parameters of the \LLM\ \citep{ding2023parameter}. In prompt tuning, the \LLM\ weights are fixed and the \LLM\ is given query-specific examples at inference time that affect its output \citep{lester2021power}. This ability to conduct in-context inference is one of the emergent abilities of \LLMs. Prompt tuning does not require large training sets. It is also preferred when query-specific examples are private or change over time, and thus can only be utilized at inference time. 

Prior works on prompt tuning mainly focus on hard prompts, which are carefully handcrafted to get the desired output. This can be time-consuming and fragile, as minor prompt modifications can lead to a significant performance drop on the downstream task \citep{suzgun2022challenging}. In contrast, \citet{zhang2022active, zhang2022automatic} and \citet{diao2023active} explored adaptive prompt design using clustering-based and uncertainty-reducing approaches. While these approaches offer some benefits, we argue that optimal designs \citep{pukelsheim2006optimal, fedorov2013theory} can outperform them by effectively balancing uncertainty and diversity. Similarly to \citet{zhang2022active, zhang2022automatic} and \citet{diao2023active}, we propose a framework for adaptive prompt design called \emph{active in-context prompt design (\aicl)}. The key idea is to design the \LLM\ prompt by adaptively choosing few-shot examples for a set of test examples at inference time. The examples are initially unlabeled and we obtain labels for the most informative ones, which maximally reduce the uncertainty in the \LLM\ prediction for all test examples. We assume that the observed labels are collected from experts (human-in-the-loop) or revealed by an oracle \citep{dasgupta05analysis, hanneke2014theory}. The focus on informativeness and diversity ensures efficient label acquisition by selecting the best examples. This reduces reliance on limited and costly resources such as expert labeling.

One motivating example for our work is \emph{theme recognition}, where the goal is to identify a unifying theme for a set of items (e.g.,  movies, grocery items, or books) provided by the user. For example, let the test query be a triplet of movie titles \say{Lion King}, \say{Jungle Book}, and \say{Tarzan}, and the goal is that the \LLM\ should infer a plausible common theme such as \say{Disney animated movies}, \say{Children's movies}, or \say{Movies with deep connections with nature}. This task is challenging due to the inherent ambiguity and many plausible themes. 
To address this, we can give the \LLM\ a few informative examples of triplets of movies and their common themes as training examples in context that can guide it towards the correct theme for the test query. 
This inherently requires a human-in-the-loop who can go over the set of triplets of movies and label their common theme for each example which can be costly. Hence, it is critical to narrow down to a few informative examples from exponentially many training examples possible for vast amounts of data like movies. 
Finally note that by exposing the \LLM\ to these training examples, we refine its understanding of the task, improve handling of ambiguity, and thus improve its ability to identify the common theme for unseen test examples. 
To address the above challenges, we propose a framework for adaptive prompt design called active in-context prompt design (\aicl). Our framework is general and can be easily extended to any active supervised-learning task, like active regression \citep{gao2011active} and active classification \citep{gao2011active}. At a high level, we treat the \LLM\ as a general inference machine \citep{brown20language, mirchandani2023large} that is shown adaptively-chosen examples with labels at inference time. The \LLM\ then utilizes them to answer any set of related test examples. The key idea is to choose the next example to label such that we maximally reduce the estimated uncertainty of the answer to the test examples. We focus on designing algorithms with the following two properties: \textbf{(1)} Implementable in any \LLM\ that can be queried efficiently. The parameters of the \LLM\ do not change or have to be observed. \textbf{(2)} Analyzable in simple models. In this work, we use linear models to motivate and analyze our algorithms. 

We now state the main contributions of our work:

\textbf{(1)} We propose a \textbf{G}-\textbf{O}ptimal design algorithm (\go). The key idea in \go\ is to retrieve the examples to label that are closest to the set of test examples in the inference task.
Our main contribution is the right notion of closeness, based on posterior covariance in a simpler model. \go\ is implementable with any \LLM\ that can be sampled from, and does not require access to model parameters, feature embeddings of the \LLM, or its gradients. 

\textbf{(2)} We propose a \textbf{S}imulation-Based \textbf{A}ctive \textbf{L}earning algorithm (\sal).  \sal\ uses simulation to estimate the impact of labeling unlabeled examples on uncertainty of the example in the inference task. \sal\ is also implementable with any \LLM\ that can be sampled from. 

\textbf{(3)} \go\ is motivated by optimal designs in linear models \citep{kiefer1960equivalence,pukelsheim2006optimal}. This allows us to analyze \go\ in linear models (\Cref{thm:go}). Our proof is a major departure from similar analyses in fixed-budget best-arm identification in bandits \citep{azizi22fixedbudget,yang22minimax}, for instance because we directly analyze the discrete allocation problem and each unlabeled example can be labeled at most once. We discuss this in detail right after \Cref{thm:go}. \sal\ is more general than \go\ because it does not make any linear model assumption in its design. We show that \sal\ and \go\ are equivalent in linear models in \Cref{thm:sal}.

\textbf{(4)} We evaluate \go\ and \sal\ on UCI \citep{uci} and OpenML \citep{OpenML2013} regression and classification tasks, custom NLP datasets, abstract reasoning corpus (ARC) tasks \citep{alford2021neurosymbolic, mirchandani2023large}, and Probabilistic Context Free Grammar (PCFG) tasks \citep{hupkes2020compositionality}. \go\ and \sal\ consistently outperform other active prompt tuning methods \citep{zhang2022active, zhang2022automatic, diao2023active} for choosing few-shot examples in majority of the tasks. 

We advance the understanding of active in-context prompt design in \LLMs\ and develop a practical methodology for adaptive prompt design. To our knowledge, this is the first paper that analyzes optimal design based prompting approaches that correctly balance uncertainty and diversity-based sampling as opposed to other existing adaptive prompting-based approaches \citep{zhang2022active, zhang2022automatic, diao2023active}.

This paper is organized as follows. \cref{sec:setting} introduces the problem setting. \cref{sec:algorithms} presents our methods and discusses their properties. \cref{sec:analysis} is devoted to analyzing our methods. \cref{sec:expt} validates our approach empirically. We review related work in detail in \cref{sec:related work}. Finally, \cref{sec:conclusions} summarizes our contributions and suggests avenues for future work.

\vspace*{-1em}
\section{Setting}
\label{sec:setting}
We pose the problem of adaptive prompt design as active learning. We adopt the following standard active learning terminology \citep{lewis1995sequential,tong2001support,dasgupta05analysis, dasgupta2007general, hanneke2014theory}. We have a $d$-dimensional \emph{feature space} $\cX \subset \R^d$ and a $d_y$-dimensional \emph{label space} $\cY \subseteq \realset^{d_y}$. A labeled example is a pair $(\bx, Y) \in \cX \times \cY$. The feature vectors and labels are related as $Y = f(\bx, \btheta_*) + \varepsilon$, where $f$ is an underlying model, $\btheta_*$ is its parameter, and $\varepsilon$ is an \emph{independent zero-mean noise vector}. Our goal is to learn to estimate $f$ on test examples by labeling training examples. 
We have a budget $T$ on the maximum number of training examples that can be labeled. This constraint can arise due to multiple reasons. For instance, human labels may be necessary and they are naturally costly. Another reason may be that the machine learning model has a limited capacity for including labeled examples, such as the length of prompts in \LLMs\ \citep{zhang2022active,zhang2022automatic,diao2023active}.

Now we introduce our notation in detail. Denote $[m] = \{1,2,\ldots,m\}$. We have $n$ \emph{training examples} $\cXt = \{\bx_1, \ldots, \bx_n \}$ and $K$ \emph{test examples} $\cXs = \set{\bx_{*, 1}, \dots, \bx_{*, K}}$. We assume that both sets are related, such as being sampled from the same distribution. 
The label of the test example $\bx_{*, k}$ is $Y_{*, k}$. In our motivating theme recognition example the training examples $\bx_i$ and test example $\bx_{*,k}$ is a concatenation of triplets of movies, and the label $Y_i$ or $Y_{*,k}$ is the common theme amongst the triplets respectively.
We want to infer $Y_{*, k}$ for all $k\in [K]$ without explicitly modeling the complex function $f$. We model the function using an \LLM\ which we treat as a general inference machine because of its large representation capacity \citep{brown2020language, mirchandani2023large}. Specifically, let $H_t = \{(X_\ell, Y_\ell)\}_{\ell \in [t - 1]}$ be a set of $t - 1$ previously labeled examples, where $X_\ell \in \cXt$ is the $\ell$-th labeled example and $Y_\ell$ is its label. Then we denote by $p(\cdot \mid \bx, H_t)$ the distribution over labels of an \LLM\ for a queried example $\bx$ when $H_t$ is used as few-shot in-context examples. To implement this in the \LLM, we simply concatenate $\bx$ and $H_t$ in context \citep{zhang2022active, zhang2022automatic, diao2023active}. We know that in-context examples affect the distribution of responses of an \LLM\ \citep{xie2021explanation, suzgun2022challenging, deng2023efficient, lee2023supervised}. So, the problem of learning $f$ under a budget $T$ can be viewed as selecting $H_{T + 1}$ such that $p(Y_{*, k} \mid \bx_{*, k}, H_{T + 1})$ is high for all test examples $k \in [K]$. This problem is challenging, especially when the training examples need to be labeled.

To effectively reduce the uncertainty of $Y_{*, k} \mid \bx_{*, k}, H_{T + 1}$, we need to quantify it. One possibility it to use the entropy $- \mathbb{E}_{Y_{*, k} \sim p(\cdot \mid \bx_{*, k}, H_{T + 1})}[\log p(Y_{*, k} \mid \bx_{*, k}, H_{T + 1})]$. This is problematic because the entropy is hard to estimate for high-dimensional random variables \citep{vershynin2020high}, especially without having access to $p(\cdot \mid \bx_{*, k}, H_{T + 1})$ beyond sampling from it. This is a shortcoming of recent adaptive prompting techniques \citep{zhang2022active, zhang2022automatic, diao2023active}. Therefore, we propose using the covariance of $Y_{*, k} \mid \bx_{*, k}, H_{T + 1}$ as the uncertainty measure. Specifically, we measure the uncertainty of the $k$-th test example by $\trace(\mathrm{cov}[Y_{*, k} \mid \bx_{*, k}, H_{T + 1}])$ and the uncertainty over all test examples by $\max_{k \in [K]} \trace(\mathrm{cov}[Y_{*, k} \mid \bx_{*, k}, H_{T + 1}])$. Since the trace of the covariance is the sum of the variances in individual dimensions, our objective can be interpreted as minimizing the maximum variance over the predicted labels of all test examples. This is a natural measure of uncertainty in linear models and corresponding optimal designs \citep{pukelsheim2006optimal, fedorov2013theory}.

Before we present our algorithms, we wanted to outline their general design. Given a budget $T$, we design sequential adaptive algorithms over $T$ rounds, where the example $X_t \in \cXt$ in round $t \in [T]$ is chosen as a function of $H_t = \{(X_\ell, Y_\ell)\}_{\ell \in [t - 1]}$ up to that round. Since $H_t$ summarizes past actions of the algorithm, we call it a \emph{history}. The label of example $X_t$ is $Y_t = f(X_t, \btheta_*) + \varepsilon_t$, where $\varepsilon_t$ is an independent zero-mean noise vector in round $t$. Our objective is to minimize the maximum uncertainty over all test examples, $\max_{k \in [K]} \trace(\mathrm{cov}[Y_{*, k} \mid \bx_{*, k}, H_{T + 1}])$.

\vspace*{-1em}
\section{Algorithms}
\label{sec:algorithms}
\begin{algorithm}[t]
  \caption{G-optimal design (\go)}
  \label{alg:go}
  \begin{algorithmic}[1]
    \STATE \textbf{Input:} Training set $\cXt = \{\bx_{i}\}_{i=1}^n $, test set $\cXs=\{\bx_{*,k}\}_{k=1}^K $, budget $T$
    \STATE $\cL_1 \gets \emptyset, \ \cU_1 \gets [n], \ H_1 \gets \{\}$
    \FOR{$t = 1, \dots, T$}
      \STATE $I_t = \argmin_{i \in \cU_t} \max_{k \in [K]} \bx_{*, k}\T
      \left(\wSigma_t^{-1} + \bx_i \bx_i\T\right)^{-1} \bx_{*, k}\,$
      \STATE $X_t \gets \bx_{I_t} \in \cXt$
      \STATE Observe label $Y_t$ of example $X_t$
      \STATE $\cL_{t + 1} \gets \cL_t \cup \{I_t\}, \
      \cU_{t + 1} \gets \cU_t \setminus \{I_t\}$
      \STATE $H_{t + 1} \gets H_t \cup \{(X_t, Y_t)\}$
    \ENDFOR
    \STATE \textbf{Output:} Sample $Y_{*,k} \sim p(\cdot \mid \bx_{*,k }, H_{T + 1})$ for $k \in [K]$
  \end{algorithmic}
\end{algorithm}

In this section, we introduce our active learning algorithms for selecting most informative training examples from $\cXt$. To simplify notation, we assume scalar labels and then discuss an extension to vector labels at the end of the section. We also let $\cL_t \subseteq [n]$ and $\cU_t \subseteq [n]$ be the indices of all labeled and unlabeled training examples up to round $t$, respectively. Note that $\cL_t \cup \cU_t = [n]$.

\subsection{Optimal Design Algorithm}
\label{sec:go}

The key idea is to label examples in $\cXt$ that minimize the maximum uncertainty of predictions over all test examples $\bx_{*, k}$. Our computation of uncertainty is borrowed from linear models. Specifically, take a linear model $Y = \bx\T \btheta_* + \varepsilon$, where $\btheta_* \in \mathbb{R}^d$ is its parameter and $\varepsilon \sim \cN(0, \sigma^2)$ is independent noise. Suppose that $\btheta_\ast \sim \cN(\btheta_0, \bSigma_0)$. Then a well-known result in Bayesian statistics \citep{bishop06pattern} is that the posterior variance of the model estimate at an example $\bx_{*, k}$ given labeled examples $H_t$ is $\bx_{*, k}\T \widehat{\bSigma}_t \bx_{*, k}$, where $\wSigma_t = (\bSigma_0^{-1} + \sigma^{-2} \sum_{\ell = 1}^{t - 1} X_\ell X_\ell\T)^{-1}$ is the posterior covariance of $\btheta_* \mid H_t$. Therefore, the maximum uncertainty over test examples is $\max_{k \in [K]} \bx_{*, k}\T \widehat{\bSigma}_t \bx_{*, k}$. The key observation is that this quantity does not depend on labels. Therefore, it can be optimized greedily by choosing the training example that minimizes it the most,
\begin{align}
  I_t
  = \argmin_{i \in \cU_t} \max_{k \in [K]} \bx_{*, k}\T
  \left(\wSigma_t^{-1} + \bx_i \bx_i\T \right)^{-1} \bx_{*, k}\,,
  \label{eq:greedy go}
\end{align}
where $\cU_t$ are indices of all unlabeled training examples up to round $t$. After the index $I_t$ is chosen, the example $\x_{I_t}$ and its label $Y_t$ are added to the history to get $H_{t + 1}$ for the next iteration $t + 1$.

This algorithm is a greedy solution to the G-optimal design \citep{pukelsheim2006optimal,katz2021improved}. We call it \textbf{G}-\textbf{O}ptimal design and abbreviate it as \go. The pseudocode of \go\ is in \cref{alg:go}. Note that \go\ does not depend on observed $Y_t$. Similar optimal designs have been effectively applied in active learning \citep{chaudhuri2015convergence,mukherjee2022chernoff}, bandits \citep{fontaine2021online, mason2021nearly}, and reinforcement learning \citep{wagenmaker2022reward}. However, this is the first paper that studies optimal design for adaptively designing prompts \citep{zhang2022active, diao2023active}.  \eqref{eq:greedy go} can be viewed as choosing that training example $\bx_i\in \cU_t$ that minimizes the maximum eigenvalue of the posterior covariance $\wSigma_t$. Therefore this leads to reducing the uncertainty over the model parameter $\btheta_*$ as the confidence ellipsoid around $\btheta_\star$ shrinks \citep{lattimore2020bandit}. 
Note that maximum eigenvalue reduction also ensures diversity as it leads to choosing training examples along all directions in $\R^d$.
%

The time complexity of \go\ is $O(K d^2 n T)$. This is because, for $T$ rounds, the algorithm searches for the best training example out of at most $n$ and evaluates it on all test examples $\bx_{*,k}\in\cXs$. The evaluation of each test example in round $t$, $\bx_{*,k}\T (\wSigma_t^{-1} + \bx_i \bx_i\T)^{-1} \bx_{*,k}$, takes $O(d^2)$ time, because $(\wSigma_t^{-1} + \bx_i \bx_i\T)^{-1}$ can be computed in $O(d^2)$ time using the Sherman-Morrison formula. In the last step, the \LLM\ is queried $K$ times to return $\{Y_{*, k}\}_{k = 1}^K$.

\subsection{Simulation-Based Algorithm}
\label{sec:sal}

\begin{algorithm}[t]
  \caption{Simulation-based active learning (\sal)}
  \label{alg:sal}
  \begin{algorithmic}[1]
    \STATE \textbf{Input:} Training set $\cXt = \{\bx_{i}\}_{i=1}^n $, test set $\cXs=\{\bx_{*,k}\}_{k=1}^K $, budget $T$
    \STATE $\cL_1 \gets \emptyset, \ \cU_1 \gets [n], \ H_1 \gets \{\}$
    \FOR{$t = 1, \dots, T$}
      \FORALL{$i \in \cU_t$}
      \FORALL{$\bx_{*,k} \in \cXs$}
        \FOR{$j = 1, 2, \ldots, m$}
          \STATE Sample $Y^{(j)}_{t, i} \sim p(\cdot \mid \bx_i, H_t)$
          \STATE $H_{t, i, j} \gets H_t \cup \{(\bx_i, Y^{(j)}_{t, i})\}$
          \STATE Sample $\tilde{Y}^{(j,1)}_{t, i, k}, \tilde{Y}^{(j,2)}_{t, i, k}
          \sim p(\cdot \mid \bx_{*,k}, H_{t, i, j})$
        \ENDFOR
        \ENDFOR
      \ENDFOR
      \STATE $\displaystyle I_t
      \gets \argmin_{i \in \cU_t} \max_{k\in [K]}\frac{1}{m} \sum_{j = 1}^m
      \left(\tilde{Y}^{(j,1)}_{t, i, k} - \tilde{Y}^{(j,2)}_{t, i, k}
      \right)^2$
      \STATE $X_t \gets \bx_{I_t} \in \cXt$
      \STATE Observe label $Y_t$ of example $X_t$
      \STATE $\cL_{t + 1} \gets \cL_t \cup \{I_t\}, \
      \cU_{t + 1} \gets \cU_t \setminus \{I_t\}$
      \STATE $H_{t + 1} \gets H_t \cup \{(X_t, Y_t)\}$
    \ENDFOR
    \STATE \textbf{Output:} Sample $Y_{*,k} \sim p(\cdot \mid \bx_{*,k }, H_{T + 1})$ for $k \in [K]$
  \end{algorithmic}
\end{algorithm}

While \go\ reduces uncertainty in label predictions, it has a major limitation. The chosen example $X_t$ at round $t$ is not affected by observed labels $(Y_\ell)_{\ell \in [t - 1]}$. This is because \eqref{eq:greedy go} does not depend on $(Y_\ell)_{\ell \in [t - 1]}$. While this is a property of linear models, it is undesirable in non-linear models, such as \LLMs. To address this limitation, we propose a new algorithm that simulates the impact of labeling examples on the uncertainty of predicted labels. We call it \textbf{S}imulation-Based \textbf{A}ctive \textbf{L}earning and abbreviated it as \sal. The pseudocode of \sal\ is provided in \cref{alg:sal}.

The key idea in \sal\ is to replace the closed-form formula in \eqref{eq:greedy go} by a simulation. We detail the algorithm next. Fix round $t$, history $H_t$, and candidate example $\bx_i$. To estimate the impact of labeling $\bx_i$, we simulate its labels $m$ times. For each simulation $j \in [m]$, we sample $Y^{(j)}_{t, i}$ from the conditional distribution $p(\cdot \mid \bx_i, H_t)$ using the \LLM. Then we extend the history $H_t$ by $\bx_i$ and its simulated label $Y^{(j)}_{t, i}$, $H_{t, i, j} = H_t \cup \{(\bx_i, Y^{(j)}_{t, i})\}$. This process results in $m$ copies of augmented histories, each reflecting a potential outcome of labeling of $\bx_i$. Finally, we take two independent samples for each $j \in [m]$ as $\tilde{Y}^{(j,1)}_{t, i, k}, \tilde{Y}^{(j,2)}_{t, i, k} \sim p(\cdot \mid \bx_{*, k}, H_{t, i, j})$. The maximum uncertainty over test examples after labeling $\bx_i$ is estimated as
\begin{align}
  \max_{k\in [K]}\frac{1}{m} \sum_{j = 1}^m \left(\tilde{Y}^{(j,1)}_{t, i, k} - \tilde{Y}^{(j,2)}_{t, i, k}
  \right)^2\,.
  \label{eq:sal score}
\end{align}
The training example with the lowest value is chosen and we denote its index by $I_t$. Then $\x_{I_t}$ and its observed label $Y_t$ are added to the history to get $H_{t + 1}$ for the next iteration $t + 1$.

Next we justify \sal. Consider the same setting as in \cref{sec:go}. Given a label $Y^{(j)}_{t, i}$ for example $\bx_i$, the posterior distribution of $\btheta_* \mid H_{t, i, j}$ is $\cN(\wtheta_{t, i, j}, \wSigma_{t, i})$, where $\wSigma_{t, i} = (\wSigma_t^{-1} + \sigma^{-2} \bx_i \bx_i\T)^{-1}$ is the simulated posterior covariance of $\btheta_*$ and
\begin{align*}
  \wtheta_{t, i, j}
  = \wSigma_{t, i} \left(\wSigma_0^{-1} \btheta_0 +
  \sigma^{-2} \left(\sum_{\ell = 1}^{t - 1} X_\ell Y_\ell + \bx_i Y^{(j)}_{t, i}\right)\right)
\end{align*}
is the posterior mean. By design, $\tilde{Y}^{(j,1)}_{t, i, k}$ and $\tilde{Y}^{(j,2)}_{t, i, k}$ are independent samples from $\cN(\bx_{*,k}\T \btheta_*, \sigma^2)$, where $\btheta_* \sim \cN(\wtheta_{t, i, j}, \wSigma_{t, i})$. Therefore, $\tilde{Y}^{(j,1)}_{t, i, k} - \tilde{Y}^{(j,2)}_{t, i, k} \sim \cN(0, 2 (\bx_{*,k}\T \wSigma_{t, i} \bx_{*,k} + \sigma^2))$. By definition, $(\tilde{Y}^{(j,1)}_{t, i, k} - \tilde{Y}^{(j,2)}_{t, i, k})^2$ is a single sample estimate of $2 (\bx_{*,k}\T \wSigma_{t, i} \bx_{*,k} + \sigma^2)$ and the sum in $\eqref{eq:sal score}$ estimates this quantity from $m$ samples. Note that this estimate is proportional to $\bx_{*,k}\T \wSigma_{t, i} \bx_{*,k}$ that appears in the G-optimal design objective in \eqref{eq:greedy go}. Therefore, in linear models, \sal\ can be viewed as an inefficient implementation of \go. This inefficiency stems from the need to simulate the \LLM.

The time complexity of \sal\ is $O(n K m T)$. This is because it searches for the best example out of at most $n$ in $T$ rounds for each test example $k \in [K]$. The evaluation of impact on each test example requires $2 m$ \LLM\ queries.

\textbf{Vector labels:} \go\ and \sal\ are easy to extend to vector labels, $d_y > 1$. \go\ does not depend on labels at all. The only modification in \sal\ is that \eqref{eq:sal score} is replaced with $  \max_{k\in [K]}\frac{1}{m}\sum_{j = 1}^m \normw{\tilde{Y}^{(j,1)}_{t, i, k} - \tilde{Y}^{(j,2)}_{t, i, k})}{2}^2$. This is the sum of the posterior variances of the labels over all dimensions.

\vspace*{-0.5em}
\section{Analysis}
\label{sec:analysis}

In this section, we analyze \go\ and \sal. The analysis is under the assumption that the labels are scalar and hence, our objective simplifies to minimizing $\max_{k \in [K]}\mathrm{var}[Y_{*,k} \mid \bx_{*,k}, H_{T + 1}]$. The analysis is organized as follows. First, we prove that our objective is decreasing in history but not supermodular, which precludes a straightforward analysis. This property of our objective function is proved in \Cref{sec:objective} and \Cref{sec:monotonicity proof}. Second, we analyze \go\ using the closed form of the posterior covariance $\wSigma_t$. Finally, we prove the equivalence of \go\ and \sal, and thereby provide guarantees for \sal. All analyses are under the assumption of a linear model with Gaussian noise. These proofs are in  \Cref{sec:go proof} and \Cref{sec:sal proof}.

\subsection{Analysis of \go}

To address challenge posed due to $f$ not being a supermodular (\cref{lem:supermodularity}), we leverage the properties of the rank-$1$ updates in \go. 
The proof is under the assumption that at round $t$, the training examples can be partitioned as $\X = S_k \cup \bS_k$. The set $S_k$ represents examples that are close to $\bx_{*,k}$. 
The set $S_k$ is convex such that for a $\alpha_k \geq 0$ we have $\bx\T \by \geq \alpha_k$ for all $\bx, \by \in S_k$. In essence, $\alpha_k$ governs the minimum level of similarity required for examples within $S_K$ to be considered similar to the test example $\bx_{*,k}$. This is achieved by setting a lower bound on the inner product between any two examples in the set.
The set $\bS_k$ represents examples that are not close to $\bx_{*,k}$. It is defined $\beta_k \geq 0$ such that $\bx\T \by \leq \beta_k$ for all $\bx \in S_k$ and $\by \in \bS_k$. In contrast to $\alpha_k$, $\beta_k$  limits the maximum similarity any example in $S_k$ can have with examples outside this set.
Define $\alpha_{\min} = \min_k \alpha_k$, and $\beta_{\max} = \beta_{\max}$. 
Define the set $S = \cap_{k=1}^K S_k$ as the set of all examples that are close to all $\{\bx_{*,k}\}_{k=1}^K$ and $\bS = \cup_{k=1}^K \bS_k$ as the set of all examples that are not close to all $\{\bx_{*,k}\}_{k=1}^K$. Assume $S \neq \{\emptyset\}$ and $|S| > T$. With this in hand, we prove the following claim.


\begin{theorem}
\label{thm:go} Let $\alpha_{\min}, \beta_{\max} \geq 0$ be set such that $\beta_{\max} \geq 1 - \alpha_{\min}^2$ and $\displaystyle T \leq \tfrac{\alpha_{\min}^2}{(\beta_{\max} + \sqrt{2}) \beta_{\max} d}$. Then for any $\bx_{*,k}$ we can show that
  $\bx_{*,k}\T \wSigma_{T + 1} \bx_{*,k}
  \leq \frac{1}{\alpha^2_{\max} T + 1} + (1 - \alpha^2_{\max})\,.$
\end{theorem}

The proof is in \cref{sec:go proof}. It is a major departure from similar proofs in active learning with a fixed budget \citep{tong2001support,hanneke2014theory,azizi22fixedbudget,yang22minimax} in three aspects. First, we analyze the discrete allocation problem in \eqref{eq:optimum} instead of its continuous optimal-design relaxation \citep{pukelsheim2006optimal}. Second, any unlabeled example in $\cX$ is labeled at most once. Finally, \eqref{eq:optimum} is asymmetric in the sense that we optimize the uncertainty of a single example $\bx_*$ over a larger set. To make the analysis manageable, we impose structure on $\cX$.
The claim in \cref{thm:go} holds for any $T$ if $\beta_{\max} = 1 / (4 d n)$ and $\alpha_{\min} = \sqrt{1 - \beta_{\max}}$. In this case, $\alpha_{\min}$ is close to $1$, and we get a near-optimal $O(1 / T)$ decrease in posterior variance.

\subsection{Analysis of \sal}

For a sufficiently large sample size $m$ in \sal, we can establish the following equivalence of \sal\ and \go.

\begin{theorem}
\label{thm:sal} Fix a failure probability $\delta\in (0,1)$. 
Define $\sigma_{t,i,k}^2 \!\!=\!\! \mathbb{E}[\frac{1}{m} \sum_{j = 1}^m
  (\widetilde{Y}^{(j, 1)}_{t, i, k} - \widetilde{Y}^{(j, 2)}_{t, i, k})^2] = 2\bx_{*,k}\T \wSigma_{t, i} \bx_{*,k} + \sigma^2$,
and define $\sigma_{t,i,\max}^2 = \max_{k\in [K]} \sigma_{t, i, k}^2$. Then  for any $t \in [T]$ and $i \in \cU_t$, we have that 
\begin{align*}
  \sigma_{t,i,\max}^2 
  \left[1 - 2\sqrt{\frac{\log(1/\delta)}{m}} \right] &\leq \max_{k \in [K]}\frac{1}{m} \sum_{j = 1}^m
  \left(\widetilde{Y}^{(j, 1)}_{t, i, k} - \widetilde{Y}^{(j, 2)}_{t, i, k}\right)^2 
  \leq  \sigma_{t,i, \max}^2\left[1 + 2\sqrt{\frac{\log(1/\delta)}{m}} +  \frac{2\log(1/\delta)}{m}\right] \,.
\end{align*} 
Moreover, for $m \geq 8\log(1/\delta)$ we have that
\begin{align*}
   2\max_k \bx_{*,k}\T \wSigma_{t, i} \bx_{*,k} + \frac{\sigma^2}{2} 
  \leq 
  \max_k \frac{1}{m}& \sum_{j = 1}^m
  \left(\widetilde{Y}^{(j, 1)}_{t, i, k} - \widetilde{Y}^{(j, 2)}_{t, i, k}\right)^2 
  \leq 
  5\max_k \bx_{*,k}\T \wSigma_{t, i} \bx_{*,k} + \frac{5\sigma^2}{2}.
\end{align*}
These claims hold with probability at least $1-\delta$.
\end{theorem}

The claim is proved in \cref{sec:sal proof}. The key idea in the proof is that \eqref{eq:sal score} multiplied by $m  / [2 (\bx_{*,k}\T \wSigma_{t, i} \bx_{*,k} + \sigma^2)]$ is a chi-squared random variable with $m$ degrees of freedom. Then we use concentration inequalities of \citet{laurent2000adaptive} to get a high-probability confidence interval on distance to the mean $m$, which in turn allows us to relate the actual variance to its empirical estimate.
\cref{thm:sal} shows that \sal\ is equivalent to \go\ for a sufficiently large sample size $m$. \cref{thm:go} can be then adapted to \sal\ as follows. The only change is in condition on $T$, which changes to
  $T
  \leq \alpha^2 \bigg/ \left(\beta + \sqrt{2} O\left(\tfrac{1 - \sqrt{1 / m}}{1 + \sqrt{1 / m}}\right)\right) \beta d + O(\sqrt{1 / m})\,.$
Therefore, \sal\ attains a near-optimal $O(1 / T)$ decrease in posterior variance as $m \to \infty$.

\vspace*{-0.5em}
\section{Experiments}
\vspace*{-0.5em}
\label{sec:expt}
We evaluate \go\ and \sal\ on a variety of prediction tasks. These tasks cover both classification and regression, including natural language features, and help us to evaluate the capabilities of \go\ and \sal\ to choose few-shot examples for active in-context prompt design. We also demonstrate that \go\ and \sal\ can be used for general pattern recognition. Detailed descriptions of all datasets are in \Cref{sec:datasets}. We describe the prompts in detail in \Cref{app:prompt-ex}.

\vspace*{-0.5em}
\subsection{Experimental Setup}
\label{sec:experimental setup}
\vspace*{-0.5em}

We use Mistral-7B \citep{jiang2023mistral}, Vicuna-13B \citep{vicuna2023}, and Falcon-40B \citep{refinedweb} as the \LLMs\ and design prompts following \citet{dinh2022lift} and \citet{suzgun2022challenging}. To investigate the impact of LLM model size on performance, we experiment with these three models of varying sizes: 7B, 13B, and 40B. Interestingly, we observe that the smaller models (Mistral-7B and Vicuna-13B) perform very poorly on certain tasks. Examples of the prompts are given in \Cref{app:prompt-ex}. Each experiment is averaged over $10$ trials. At the beginning of each trial, we randomly select $K=20$ test examples. We describe in detail how the training set and $n$ are chosen for each dataset in \Cref{app:addl_expt}.

Each run is a simulation that proceeds as follows. In round $t$, each method selects a training example to label $X_t$ and then observes the true label $Y_t$. All past interactions $H_t = \set{(X_\ell, Y_\ell)}_{\ell \in [t - 1]}$ along with the test examples $\bx_{*, k}$ are used to craft a prompt for the \LLM. The performance at round $t$ is evaluated by the \emph{error} $\cL(t) = \frac{1}{K} \sum_{k = 1}^K \cL(Y_{*, k}, \tilde{Y}_{*, k, t})$, where $Y_{*, k}$ is the true label of test example $\bx_{*, k}$, $\tilde{Y}_{*, k, t}$ is its \LLM\ predicted label in round $t$, and $\cL(y_*, y)$ is a task-specific error function. For classification tasks, we choose $\cL(y_*, y) = \indic{y_* = y}$ and call $\cL(t)$ a \emph{misclassification error}. For regression tasks, we choose $\cL(y_*, y) = (y_* - y)^2$ and call $\cL(t)$ the \emph{MSE}. For pattern recognition tasks, where $Y_{*, k}$ and $\tilde{Y}_{*, k, t}$ are either vectors or matrices, we choose let $\cL(y_*, y) = \indic{y_* = y}$ and $\cL(t)$ represents \emph{0-1 error}.

We posit that \go\ and \sal\ perform well because they both reduce the uncertainty of test examples based on the right notion of similarity. To show this, we compare to baselines that reduce uncertainty uniformly (like \unif), or reduce uncertainty informatively (\least\ or \maxent), or only select examples with similar features to test examples (\greedy). As shown in our extensive experiments, these baselines fail to match the capabilities of \go\ and \sal\ to select informative examples in the majority of the tasks. The following methods are compared in our experiments:

\noindent \textbf{(1)} \unif: The example $X_t$ in round $t$ is sampled uniformly at random from the unlabeled set $\cU_t$. \unif\ is a pure exploration algorithm that does not take into account the similarity to test examples and variance reduction. We chose it as a baseline because it tends to work well in practice. Therefore, it is used frequently in active learning and prompt tuning papers \citep{zhang2022automatic, diao2023active}.

\noindent \textbf{(2)} \greedy: The example $X_t$ in round $t$ is chosen to align the most with all test examples $\bx_{*,k}$ such that  $I_t \gets \argmax_{i \in \cU_t} \max_{k \in [K]} \bx_{*, k}\T \bx_i$. This baseline shows that our information gathering rule in \eqref{eq:greedy go} goes beyond pure feature similarity. This baseline is similar to the automatic exemplar construction method by clustering by \citet{zhang2022automatic}.


\noindent \textbf{(3)} \least: This is similar to the disagreement-based method of \citet{diao2023active}. The disagreement score of the example $i\in\U_t$ is calculated as $s_{i} = \sum_{k=1}^K Y_{tik}$ where $Y_{tik}\sim p\left(\cdot \mid \bx_{*,k}, \bx_i\right)$ is the number of unique answers by for test example $\bx_{*,k}$ using only $\bx_i$ as the in-context example by the \LLM. Then the example selected at round $t$ is $I_t \leftarrow \argmax_{i\in\U_t}s_i$. This is the unlabeled example where the \LLM, disagrees the most for all test examples and is least confident. We compare against this baseline to show that our information gathering rule in \eqref{eq:greedy go} goes beyond just uncertainty sampling but also takes into account the diversity of training examples when choosing to label the next example.

\noindent \textbf{(4)} \maxent: This is the uncertainty-based maximum entropy method of \citet{zhang2022active, diao2023active}. At round $t$ the example with the highest entropy is selected as $I_t \leftarrow \argmax_{i\in\U_t}-\sum_{k=1}^K \bar{Y}_{tik}\ln \bar{Y}_{tik}$ where $\bar{Y}_{tik} \sim p\left(\cdot \mid \bx_{*,k}, \bx_i \right)$ is the frequency of a predicted answer among all predictions for the test example $\bx_{*,k}$ using $\bx_{i}$ as the in-context example by the \LLM. 
%
A larger entropy denotes greater uncertainty and therefore, an unlabeled example with the largest entropy will be selected. 
Again we compare against this uncertainty-based baseline to show that our information gathering rule in \eqref{eq:greedy go} goes beyond just uncertainty sampling but also considers the diversity of training examples when choosing to label the next example.

\noindent \textbf{(5)} \go\ (ours): This is \cref{alg:go} where $\cX$ are the original feature vectors.

\noindent \textbf{(6)} \goi\ (ours): This is \cref{alg:go} where the original feature vectors are used in the prompt but $\cX$ are their $768$-dimensional Instructor embeddings \citep{INSTRUCTOR}. We use this for natural language classification tasks.

\textbf{(7)} \sal\ (ours): This is \cref{alg:sal} where $\cX$ are the original feature vectors. To implement \sal\ efficiently, we combine it with \go\ as a preprocessing step. Specifically, in round $t$, \go\ first chooses $5$ most informative examples from $\cU_t$ and then we apply \sal. We use $m = 1$ in all experiments. We use these approximations because \sal\ is computationally expensive (\cref{sec:sal}). Similarly to \goi, we use Instructor embeddings for natural language classification tasks. 

\begin{table}[t!]
\begin{adjustbox}{width=0.8\columnwidth,center}
    \hspace*{-2em}\begin{tabular}{c|c|c|c|c|c|c|c}
&\textbf{Datasets} &\unif & \greedy & \least & \maxent & \go\ (ours) & \sal\ (ours)\\\hline
 & iris & $0.41 \pm 0.11$ & $0.60 \pm 0.13$ & $0.64 \pm 0.15$ & $0.72 \pm 0.17$ & $0.38 \pm 0.14$ & $\mathbf{0.34 \pm 0.14}$\\
M & banknote & $0.75 \pm 0.10$ & $\mathbf{0.58 \pm 0.04}$ & $0.59 \pm 0.02$ & $0.73 \pm 0.16$ & $0.77 \pm 0.07$ & $0.75 \pm 0.15$\\
 & balance-scale & $0.61 \pm 0.13$ & $0.69 \pm 0.22$ & $0.55 \pm 0.25$ & $0.57 \pm 0.14$ & $\mathbf{0.48 \pm 0.09}$ & $0.72 \pm 0.04$\\
 & thyroid-new & $\mathbf{0.44 \pm 0.12}$ & $0.70 \pm 0.08$ & $0.74 \pm 0.12$ & $0.57 \pm 0.08$ & $0.55 \pm 0.06$ & $0.63 \pm 0.14$\\\hline
& iris & $0.22 \pm 0.24$ & $0.60 \pm 0.37$ & $0.60 \pm 0.49$ & $0.40 \pm 0.20$ & $0.20 \pm 0.24$ & $\mathbf{0.20 \pm 0.24}$\\
V & banknote & $0.40 \pm 0.37$ & $0.80 \pm 0.24$ & $0.50 \pm 0.32$ & $0.50 \pm 0.32$ & $0.50 \pm 0.32$ & $\mathbf{0.10 \pm 0.20}$\\
 & balance-scale & $0.60 \pm 0.20$ & $0.60 \pm 0.37$ & $0.50 \pm 0.32$ & $0.80 \pm 0.24$ & $\mathbf{0.30 \pm 0.24}$ & $0.50 \pm 0.00$\\
 & thyroid-new & $0.52 \pm 0.45$ & $1.00 \pm 0.00$ & $0.70 \pm 0.24$ & $0.50 \pm 0.00$ & $0.60 \pm 0.20$ & $\mathbf{0.50 \pm 0.32}$\\\hline
& iris & $\mathbf{0.20 \pm 0.06}$ & $0.62 \pm 0.14$ & $0.70 \pm 0.20$ & $0.65 \pm 0.18$ & $0.42 \pm 0.10$ & $0.33 \pm 0.23$\\
 F & banknote & $0.45 \pm 0.23$ & $0.53 \pm 0.25$ & $0.60 \pm 0.12$ & $0.42 \pm 0.17$ & $0.45 \pm 0.06$ & $\mathbf{0.45 \pm 0.1}$\\
 & balance-scale & $0.70 \pm 0.28$ & $0.68 \pm 0.13$ & $0.85 \pm 0.12$ & $0.62 \pm 0.08$ & $0.47 \pm 0.24$ & $\mathbf{0.45 \pm 0.13}$\\
 & thyroid-new & $0.55 \pm 0.29$ & $0.57 \pm 0.20$ & $0.75 \pm 0.19$ & $0.65 \pm 0.15$ & $0.55 \pm 0.23$ & $\mathbf{0.53 \pm 0.12}$
\end{tabular}
\end{adjustbox}
    \caption{Misclassification error in classification datasets using Mistral-7B (M), Vicuna-13B (V), and Falcon-40B (F) on $K=20$ test examples at the end of budget $T=5$.}
    \label{tab:Class-1}
    \vspace*{-1em}
\end{table}

\begin{table}[t!]
\begin{adjustbox}{width=0.8\columnwidth,center}
    \hspace*{-2em}\begin{tabular}{c|c|c|c|c|c|c|c}
&\textbf{Datasets} &\unif & \greedy & \least & \maxent & \go\ (ours) & \sal\ (ours)\\\hline
&machine(e+04)& $11.4 \pm 3.34$ & $10.5 \pm 2.44$ & $14.3 \pm 3.39$ & $11.0 \pm 1.93$ & $\mathbf{10.5 \pm 3.74}$ & $10.6 \pm 3.84$\\
M &fifa(e-04)& $1.40 \pm .216$ & $3.72 \pm 1.12$ & $1.18 \pm .53$ & $4.15 \pm 1.11$ & $.999\pm .404$ & $\mathbf{.68 \pm .26}$\\\hline
&machine(e+04)& $5.59 \pm 1.35$ & $5.04 \pm .851$ & $7.95 \pm 1.69$ & $4.98 \pm 1.06$ & $5.66 \pm 1.54$ & $\mathbf{4.80 \pm 1.46}$\\
V &fifa(e+03)& $5.90 \pm 1.59$ & $4.72 \pm .931$ & $5.11 \pm 1.12$ & $6.76 \pm 1.69$ & $\mathbf{1.44 \pm .258}$ & $2.59 \pm .742$\\\hline
& machine(e+03) & $1.16 \pm 1.22$ & $4.28 \pm 2.30$ & $2.15 \pm 1.08$ & $3.50 \pm 2.45e+03$ & $.32 \pm .209$ & $\mathbf{2.96 \pm 1.56}$\\
F & fifa(e+01) & $7.78 \pm 3.85$ & $6.95 \pm 2.93$ & $12.4 \pm 12.3$ & $26.3 \pm 37.6$ & $7.90 \pm 4.64$ & $\mathbf{4.61 \pm 4.35}$\\
\hline
\end{tabular}
\end{adjustbox}
    \caption{MSE in regression datasets using Mistral-7B (M), Vicuna-13B (V), and Falcon-40B (F) on $K=20$ test examples at the end of budget $T=5$.}
    \label{tab:Reg-1}
    \vspace*{-1em}
\end{table}

\begin{table}[t!]
    \begin{adjustbox}{width=0.8\columnwidth,center}
\begin{tabular}{c|c|c|c|c|c|c}
\textbf{Task} &\unif & \greedy & \least & \maxent & \go\ (ours) & \sal\ (ours)\\\hline
Arc-1 & $0.45 \pm 0.50$ & $0.45 \pm 0.50$ & $0.90 \pm 0.30$ & $0.60 \pm 0.49$ & $0.30 \pm 0.46$ & $\mathbf{0.15 \pm 0.36}$\\
 Arc-2 & $0.80 \pm 0.40$ & $1.00 \pm 0.00$ & $0.80 \pm 0.40$ & $0.80 \pm 0.40$ & $0.80 \pm 0.40$ & $\mathbf{0.01 \pm 0.01}$\\
 PCFG-1 & $0.60 \pm 0.49$ & $1.00 \pm 0.00$ & $1.00 \pm 0.00$ & $1.00 \pm 0.01$ & $0.20 \pm 0.40$ & $\mathbf{0.02 \pm 0.01}$\\
 PCFG-2 & $0.20 \pm 0.40$ & $1.00 \pm 0.00$ & $0.20 \pm 0.40$ & $1.00 \pm 0.00$ & $0.20 \pm 0.40$ & $\mathbf{0.14 \pm 0.40}$
\end{tabular}
\end{adjustbox}
    \caption{0-1 error using Falcon-40B on $K=20$ test examples at the end of budget $T=5$. ARC-1 is the expansion-contraction task, ARC-2 is the rotation task, PCFG-1 is the add-subtract task, and PCFG-2 is the repeat experiment task. Mistral-7B and Vicuna-13B perform very poorly on these tasks and thus are omitted.}
    \label{tab:arc-pcfg}
    \vspace*{-2.2em}
\end{table}

\begin{table}[t!]
\begin{adjustbox}{width=0.8\columnwidth,center}
    \hspace*{-2em}\begin{tabular}{c|c|c|c|c|c|c|c}
&\textbf{Datasets} &\unif & \greedy & \least & \maxent & \go\ (ours) & \sal\ (ours)\\\hline
 & movie & $0.32 \pm 0.17$ & $0.90 \pm 0.06$ & $0.87 \pm 0.09$ & $0.55 \pm 0.18$ & $\mathbf{0.27 \pm 0.10}$ & $0.49 \pm 0.11$\\
M & entity & $0.69 \pm 0.19$ & $0.86 \pm 0.06$ & $0.87 \pm 0.09$ & $0.59 \pm 0.15$ & $0.65 \pm 0.18$ & $\mathbf{0.39 \pm 0.19}$\\
 & theme & $0.74 \pm 0.05$ & $0.74 \pm 0.09$ & $0.82 \pm 0.16$ & $\mathbf{0.68 \pm 0.09}$ & $0.80 \pm 0.09$ & $0.81 \pm 0.08$\\\hline
& movie & $0.10 \pm 0.20$ & $0.70 \pm 0.24$ & $0.90 \pm 0.20$ & $0.30 \pm 0.24$ & $\mathbf{0.02 \pm 0.01}$ & $0.10 \pm 0.20$\\
V & entity & $0.20 \pm 0.24$ & $0.90 \pm 0.20$ & $0.70 \pm 0.24$ & $0.60 \pm 0.20$ & $\mathbf{0.10 \pm 0.20}$ & $0.10 \pm 0.20$\\
 & theme & $0.90 \pm 0.20$ & $0.70 \pm 0.40$ & $1.00 \pm 0.00$ & $0.70 \pm 0.24$ & $\mathbf{0.60 \pm 0.37}$ & $0.80 \pm 0.40$\\\hline
& movie & $0.55 \pm 0.06$ & $0.62 \pm 0.18$ & $0.78 \pm 0.17$ & $0.53 \pm 0.22$ & $\mathbf{0.38 \pm 0.18}$ & $0.47 \pm 0.17$\\
F & entity & $0.55 \pm 0.23$ & $0.62 \pm 0.08$ & $0.75 \pm 0.14$ & $0.65 \pm 0.24$ & $0.47 \pm 0.23$ & $\mathbf{0.42 \pm 0.24}$\\
 & theme & $0.68 \pm 0.20$ & $0.70 \pm 0.13$ & $0.85 \pm 0.05$ & $0.85 \pm 0.09$ & $\mathbf{0.53 \pm 0.18}$ & $0.55 \pm 0.19$\\\hline
\end{tabular}
\end{adjustbox}
    \caption{Misclassification error in natural language classification 
    tasks using Mistral-7B (M), Vicuna-13B (V), and Falcon-40B (F) on $K=20$ test examples at the end of budget $T=5$.}
    \label{tab:NLC-1}
    \vspace*{-2.2em}
\end{table}

All used datasets and experimental setups are described in \Cref{app:addl_expt}. This section only summarizes the main results.

\textbf{Standard classification and regression tasks.} We use $4$ classification and $2$ regression datasets from UCI and OpenML (\Cref{sec:datasets}). We set $T=5$ to simulate the realistic scenario when the test queries provided by the user need to be inferred quickly. 
For classification tasks, $K=20$ test examples are chosen among the different classes of the dataset. We describe in detail how the training set and $n$ are chosen for each dataset in \Cref{app:addl_expt}.
For regression tasks, $K=20$ random test examples are chosen. Our results on classification tasks are reported in \cref{tab:Class-1} and on regression tasks in \cref{tab:Reg-1}. We observe that \go\ and \sal\ are the best-performing methods in the majority of the datasets. Note that there is no single baseline that consistently outperforms them.

\textbf{General pattern recognition.} We experiment with $4$ tasks: ARC expansion and contraction, ARC rotation, PCFG Add-Subtract, and PCFG Repeat. Both inputs and outputs in these tasks are vectors or matrices. We describe examples of ARC and PCFG tasks in detail in \Cref{sec:datasets}. Each dataset comprises examples of two patterns: expansion and contraction, clockwise and counter-clockwise rotation, add and subtract, repeat first and second digits. In each trial, we choose $K = 20$ different test examples equally from two patterns and set $T=5$. Our results are reported in \cref{tab:arc-pcfg}. In all tasks, \go\ and \sal\ are the best-performing methods. \sal\ outperforms \go\ in ARC \citep{alford2021neurosymbolic, mirchandani2023large} and PCFG \citep{hupkes2020compositionality} consistently. 

\textbf{Natural language classification tasks (NLC)}. We show that \go\ and \sal\ work well on general NLP tasks where no explicit numerical features are available. We create three synthetic datasets based on the following tasks: (i) \texttt{movie-names}: predicting a genre from a movie name (e.g., romance, horror), (ii) \texttt{movie-theme}: predicting a common theme for a pair of movie names (e.g., coming-of-age, sci-fi), and (iii) \texttt{entity-names}: predicting an entity's type from its name (e.g., celebrity, mountain, river). Each dataset comprises $5$ classes. Further details regarding the additional datasets are provided in \Cref{app:addl_expt}. In each trial, $K=20$ test examples were randomly chosen across the five classes. The feature vectors in \go\ and \sal\ are Instructor embeddings of the original text features. Our results are reported in \cref{tab:NLC-1}. We observe again that \go\ and \sal\ are the best-performing methods in the majority of the datasets. There is no single baseline that consistently outperforms them. This shows that the optimal design-based approach of \go\ and \sal\ correctly balances uncertainty and diversity-based sampling.

\vspace*{-0.5em}
\section{Conclusions}
\vspace*{-0.5em}
\label{sec:conclusions}
In this paper, we studied the framework of active in-context prompt design (\aicl) that uses optimal design to systematically choose the most informative unlabeled examples to label for a set of test examples. 
These informative labeled examples are then used to minimize the prediction error of the \LLM\ for all the test examples. 
To our knowledge, this is the first paper that studies optimal design for adaptive prompt design.
Inspired by the linear model, we proposed an algorithm \go\ that strategically chooses the most informative examples that minimize the variance of the posterior covariance for any test example from the test set.
We proposed a second algorithm \sal\ that uses simulations to estimate the impact of how unlabeled examples reduce \LLM uncertainty for all test examples. It then chooses to label examples that maximally reduce the uncertainty of the \LLM\ for all test examples from the test set. 
We theoretically analyze \go\ and \sal\ and show their equivalence in linear models. We show that both algorithms guarantee information gain at each iteration.
Finally, we show empirically that, when used with \LLMs\ like Mistral-7B, Vicuna-13B,  and Falcon-40B, both \go\ and \sal\ result in better prediction accuracy than other baselines \citep{zhang2022active, zhang2022automatic, diao2023active} on tasks like classification, regression, ARC, PCFG, and natural language generation. 
Our research opens up exciting new directions for future work such as extending \aicl\ framework  beyond text to enable informative example selection for tasks involving images, videos, or other modalities using multi-modal \LLMs\ \citep{yin2023survey}. Additionally, the integration of active learning with diffusion models, a powerful class of generative models, presents promising directions for future research \citep{ho2020denoising}.
%



\bibliographystyle{plainnat}
\bibliography{biblio,References}

\textbf{Broader Impact:} We propose a framework for adaptive prompt design for Large Language Models (\LLMs). Our method does not address the fundamental ethical and societal challenges inherent to these models. Specifically, it does not correct the biases or discriminatory outputs that may exist in the \LLMs. Additionally, our approach does not solve the issue of \LLMs\ generating reasonable but inaccurate information, known as hallucinations. Therefore, while our method improves the functionality of \LLMs, it is essential for users and developers to use these models responsibly and remain aware of their limitations, particularly in terms of content accuracy and fairness.

\clearpage
\onecolumn
\appendix

\section{Related Work}
\label{sec:related work}

We study the problem of choosing human demonstrations adaptively to get the desired output from the \LLM\ as quickly as statistically possible. We use active learning to choose them and then ask a human to label them. Finally, the human demonstrations are fed as in-context input to the \LLM\ together with the main user query, to obtain the desired output. Thus the name is active transductive inference. Prior works on prompt-tuning and transductive inference \citep{lester2021power,dong2022survey,zhang2022active,min2022rethinking,wu2022self,yu2022multimodal,suzgun2022challenging,liu2023pre,yu2023fusing, liu2023hierarchical} focus on hard prompt-tuning where the user must carefully handcraft the prompt to get the desired output for tasks like movie recommendation, disambiguation QA, navigation, etc. Such examples of carefully handcrafting hard prompts with demonstrations can be found in \citet{suzgun2022challenging, srivastava2022beyond}. These papers also show how failing to design such prompts with demonstrations can lead the \LLM\ to predict wrong outputs. Note that we adaptively design the prompt through carefully chosen demonstrations. In our experiments, we show significant improvement over randomly chosen demonstrations.

The problem of active learning is also related to \emph{dataset augmentation} \citep{dukler2021diva}. In this work, the most informative unlabeled examples are chosen by optimizing the error of the model on the validation set. The gradient of the validation set error with respect to the weights of unlabeled examples has a closed form when the original model is linearized. The main difference in active learning, including in our work, is that labeling all unlabeled examples would be costly. Therefore, the labels are not available in advance and are queried adaptively. We also learn in context and do not assume that the gradient information of the \LLM\ is available. Prompt composition has also been an active area of research. \citet{bowman2023carte} proposed a-la-carte prompt tuning, where prompts are tuned on individual datasets and composed at inference time to mimic the performance of the model that would have been trained on the union of the corresponding datasets. This idea has been further extended by \citet{perera2023prompt}, where prompts for previously unseen tasks are obtained by linearly combining prompts from known tasks. To do this, they use spectral decomposition and project prompts from known tasks to a lower dimensional space. In our work, we do not tune prompts or compose simpler models. We actively probe an \LLM, treated as a black box without any extra side information, to answer a test example as accurately as possible, with as little variance in the answer as possible.

Recently there has been a lot of progress in prompt tuning (or aligning). \citet{hassan2023align} studies prompt aligning for a single test sample to adapt to multi-modal test prompts at test time by minimizing the feature distribution shift to the test domain. In contrast in this paper, we study adapting prompts for many test samples without the feature distribution shift assumption. The \cite{wang2023large} trains a smaller \LLM to select demonstrations for a larger \LLM. However, we rely on active learning to select the smallest number of informative prompts to be labeled by human labelers. This avoids finetuning a smaller \LLM\ for individual tasks. \citet{wang2023context} studies transductive inference for diffusion models for a different setting where given a pair of task-specific example images, such as depth from/to image and scribble from/to image, and text guidance, the model automatically understands the underlying task and performs the same task on a new query image following the text guidance. However, in our setting, we do not explicitly encode any guidance text. The \citet{wen2023hard} proposes to mix the nearest neighbor method with gradient optimization to select prompts during test time. Similarly, \citet{zhang2023makes} proposes a nearest neighbor approach to select in-context examples for computer vision tasks. We compare our approach against such nearest-neighbor selection algorithms. Finally, \citet{bai2023transformers, wang2023universality} analyze transductive inference theoretically to understand its universality, generalization capability, and limitations. In contrast, we only do a theoretical analysis of \aicl\ to show it maximally reduces the estimated variance of the answer to the user’s query. The \citet{zhang2022automatic} study the chain-of-thought prompting where they automatically select the prompt using a clustering-based approach.

There are some related works in the area of medical diagnosis chatbot examples that we shared in the introduction. One notable study \citet{CARUCCIO2024121186}, although not utilizing machine learning techniques, provides valuable insights with its implementation of three hardcoded prompt designs for accurate diagnosis. These designs, however, do not engage in active learning, as they lack the capability to adapt based on user input. Similarly, \citet{Tomoyuki2023} gives more insights into self-diagnostics of orthopedic diseases using ChatGPT. In contrast, web applications such as Buoy Health \cite{BuoyHealth} and Live Healthily \cite{Livehealthily} employ a more dynamic approach, actively tailoring subsequent questions based on users' responses. This aligns with active learning principles but it is not clear what techniques they apply and is notably underexplored in academic literature, indicating a potential area for further research. Our setting also goes beyond the single shot active prompt tuning studied in \citet{margatina2023active}. Note that this work studies prompt tuning only for one iteration, and does not take into account the historical context. So it has limited ability for complex tasks like ARC \citep{alford2021neurosymbolic} and PCFG \citep{hupkes2020compositionality} as well as handling vector labels like \go\ and \sal. 


\textbf{Active Learning (\AL):} Recently, there has been a lot of focus on using deep \AL\ to finetune \LLMs. All \AL\ algorithms tend to balance uncertainty and diversity in the selection of unlabeled examples. We briefly discuss them below and also highlight the main difference of these approaches with prompt aligning with \go\ and \sal.

\par\textbf{(1)} \core: This is a pure diversity-based approach using a coreset selection. In every iteration, first, the embedding of each unlabeled example is computed from the network's penultimate layer, and then unlabeled examples are selected using a greedy furthest-first traversal conditioned on all labeled examples \citep{sener2017active, geifman2017deep, citovsky2021batch}. Observe that in our setting we do not have access to the penultimate layer of the \LLM.
\par\textbf{(2)} \least: This is an uncertainty-based active learning algorithm. Here, the uncertainty score of an unlabeled example is its predicted class probability. 
At every iteration, this algorithm then samples unlabeled examples with the smallest uncertainty scores \citep{settles2009active, settles2011theories, wang2014new}.
    %
\par\textbf{(3)} \margin: This is also an uncertainty-based active learning algorithm \citep{tong2001support, balcan2009agnostic, settles2009active}. At every iteration $t$ it selects unlabeled examples that are sorted according to their multiclass margin score and then selects unlabeled examples that are the hardest to discriminate and can be thought of as examples closest to their class margin. However, in our setting, we do not have any information on the hypothesis space of the \LLM\ and hence cannot implement such a baseline.
\par\textbf{(4)} \entropy: This is also an uncertainty-based active learning algorithm \cite{wang2014new, kremer2014active, diao2023active}. At every iteration $t$ it selects unlabeled examples according to the entropy of the example's predictive class probability distribution. 
We show that \greedymu\ is outperformed significantly by \go\ and \sal\ in the prediction, pcfg, or arc tasks.
    %
\par\textbf{(5)} \badge: This is an algorithm that combines both uncertainty and diversity sampling \citep{ash2019deep, ash2021gone}. For each unlabeled example $\bx$ its gradient embedding $g_{\bx}$ is computed with respect to the parameters of the model's penultimate layer. Finally, \badge\ chooses a batch of samples to sample by applying $k$-Means++ \citep{arthur2006k} on the gradient embeddings. Again recall that we cannot implement such a baseline as we do not have access to the \LLMs\ last layer.

\par\textbf{(6)} \badgekm: This algorithm is similar to \badge\ but in the final step instead of $k$-Means++ it uses $k$-Means on the gradient embeddings. In \citet{yuan2020cold} it is observed that applying $k$-Means on the embeddings results in an increase in accuracy over baselines in some datasets. Further \citet{yuan2020cold} observed from the t-SNE plots that $k$-Means select centers that are further apart compared to the ones chosen by $k$-Means++ which leads to more diverse sampling in batches.
    
\par\textbf{(7)} \bald: Bayesian Active Learning by Disagreements  \citep{kirsch2019batchbald, pmlr-v70-gal17a} chooses unlabeled examples that are expected to maximize the information gained from the model parameters $\btheta_t$, i.e. the mutual information between predictions and model posterior. 

A more comprehensive survey on how \AL\ is used for finetuning deep models can be found in \citet{ren2021survey, zhan2022comparative}. The \citet{bhatt2024experimental} study how experimental design can be used to select prompts for finetuning a pre-trained \LLM. 
Some recent works have also focused on selecting unlabeled examples only within a task \citet{wei2021finetuned, chen2023active, fifty2021efficiently}. However, these works are geared towards selecting prompts within a task for finetuning, whereas we focus on adaptive prompt design using experimental design. The work of \citet{perlitz2023active} also uses \AL\ for finetuning prompts for \LLMs to improve label efficiency. The \citet{kung2023active} proposes an \AL\ framework for instruction tuning. However, their approach again focuses on selecting unlabeled examples inside each task and discriminating one task from another. However, they make the simplifying assumption that all unlabeled examples inside the tasks are equally informative which may inhibit the quality of the selected subset.

\section{Proofs}
\label{app:proofs}
This section contains the properties of our objective and proofs of our main claims.

\subsection{Properties of our objective}
\label{sec:objective}

By the total variance decomposition for a linear model with observation variance $\sigma^2$, we get
\begin{align*}
  \max_{k \in [K]}\mathrm{var}[Y_{*,k} \mid \bx_{*,k}, H_{T + 1}] 
  &= 
  \max_{k \in [K]} \mathrm{var}[\mathbb{E}[Y_{*,k} \mid \bx_{*,k}, \theta_*, H_{T + 1}]
  \mid \bx_{*,k}, H_{T + 1}] \\
  & \quad +
  \max_{k \in [K]} \mathbb{E}[\mathrm{var}[Y_{*,k} \mid \bx_{*,k}, \theta_*, H_{T + 1}]
  \mid \bx_{*,k}, H_{T + 1}] \\
   &= 
  \max_{k \in [K]} \sigma^2 \bx_{*,k}\T \wSigma_t^{-1} \bx_{*,k} + \sigma^2\,.
\end{align*} 
Therefore, the variance minimization of $\max_{k \in [K]}Y_{*,k} \mid \bx_{*,k}, H_{T + 1}$ is a combinatorial optimization problem. It can be formulated as follows
\begin{align}
  \textstyle
  S_*
  = \argmin_{S \subseteq [n], |S| = T} f(S)\,,
  \label{eq:optimum}
\end{align}
where the function $f(S) =  \max_{k \in [K]}\bx_{*,k}\T \left(\bSigma_0^{-1} + \sigma^{-2} \sum_{i \in S} \bx_i \bx_i\T\right)^{-1} \bx_{*,k}$ represents the uncertainty associated with any subset $S$, and $S_*$ denotes the optimal subset of size $T$. If $f(S)$ was monotone and supermodular in $S$, we would have $(1 - 1 / e)$-optimality guarantees for \go\ \citep{nemhauser78approximation}. We start with proving the monotonicity of $f(S)$.

\begin{lemma}
\label{lem:monotonicity} Function $f(S)$ is monotonically decreasing.
\end{lemma}

The claim is proved in \cref{sec:monotonicity proof}. Now we state the definition of a supermodular function and show that $f(S)$ is not supermodular. 

\begin{definition}[Supermodular function]
Let $[n]$ be a set of $n$ elements. A function $f: 2^{[n]} \to \realset$ is supermodular if it satisfies the property of diminishing marginal returns. For any $S_1 \subseteq S_2 \subseteq [n]$ and $\bx \notin S_2$, we have $f(S_1 \cup \set{\bx}) - f(S_1) \leq f(S_2 \cup \set{\bx}) - f(S_2)$.
\end{definition}

\begin{lemma}
\label{lem:supermodularity} The function $f(S)$ is not supermodular.
\end{lemma}
\begin{proof}
Take $\bSigma_0 = I_d$, fix $K=1$, $\bx_1 = (1 / \sqrt{2}, 1 / \sqrt{2})$, $\bx_2 = (0, 1)$, and $\bx_\ast = (1, 0)$. Then $f(\set{2}) - f(\emptyset) = 0 < 0.035714 \approx f(\set{1, 2}) - f(\set{1})$.
\end{proof}

\subsection{Proof of \cref{lem:monotonicity}}
\label{sec:monotonicity proof}

Recall that the objective function is $f(S) = \max_{k\in [K]}\bx_{*,k}\T \left(\bSigma_0^{-1} + \sigma^{-2} \sum_{i \in S} \bx_i \bx_i\T\right)^{-1} \bx_{*,k}$. Without loss of generality, let $\sigma^2 = 1$. Fix any $j \in [n] \setminus S$ and let $S_+ = S \cup \set{j}$. Then the objective value at $S_+$ can be written as
\begin{align*}
  f(S_+)
  = \max_{k\in [K]}\bx_{*,k}\T \left(\bSigma_0^{-1} + \sum_{i \in S_+} \bx_i \bx_i\T\right)^{-1} \bx_{*,k}
  = \max_{k\in [K]}\bx_{*,k}\T \left(\bA + \bx_j \bx_j\T\right)^{-1} \bx_{*,k}\,,
\end{align*}
where $\bA = \bSigma_0^{-1} + \sum_{i \in S} \bx_i \bx_i\T$. By the Sherman–Morrison formula, we have
\begin{align*}
  \left(\bA + \bx_j \bx_j\T\right)^{-1}
  = \bA^{-1} - \frac{\bA^{-1} \bx_j \bx_j\T \bA^{-1}}{1 + \bx_j\T \bA^{-1} \bx_j}\,.
\end{align*}
Now note that $\displaystyle \frac{\bA^{-1} \bx_j \bx_j\T \bA^{-1}}{1 + \bx_j\T \bA^{-1} \bx_j}$ is a positive semi-definite matrix for any $\bx_j$ and any positive semi-definite matrix $\bA$. As a result, $\displaystyle \bx_{*,k}\T \frac{\bA^{-1} \bx_j \bx_j\T \bA^{-1}}{1 + \bx_j\T \bA^{-1} \bx_j} \bx_{*,k} \geq 0$ holds for any $\bx_{*,k}$ and thus
\begin{align*}
  f(S_+)
  = \max_{k\in [K]}\bx_{*,k}\T \bA^{-1} \bx_{*,k}\T -
  \bx_{*,k}\T \frac{\bA^{-1} \bx_j \bx_j\T \bA^{-1}}{1 + \bx_j\T \bA^{-1} \bx_j} \bx_{*,k}
  \leq \max_{k\in [K]}\bx_{*,k}\T \bA^{-1} \bx_{*,k}\T
  = f(S)\,.
\end{align*}
This proves our claim.

\subsection{Proof of \cref{thm:go}}
\label{sec:go proof}

The proof is under the assumption that at round $t$, the training examples can be partitioned as $\cXt = S_k \cup \bS_k$. The set $S_k$ represents examples that are close to $\bx_{*,k}$. 
%
%
The set $S_k$ is convex and define $\alpha_k \geq 0$ such that $\bx\T \by \geq \alpha_k$ for all $\bx, \by \in S_k$.
The set $\bS_k$ represents examples that are not close to $\bx_{*,k}$. It is defined $\beta_k \geq 0$ such that $\bx\T \by \leq \beta_k$ for all $\bx \in S_k$ and $\by \in \bS_k$. Define $\alpha_{\min} = \min_k \alpha_k$, and $\beta_{\max} = \beta_{\max}$. 

Define the set $S = \cap_{k=1}^K S_k$ as the set of all examples that are close to all $\{\bx_{*,k}\}_{k=1}^K$ and $\bS = \cup_{k=1}^K \bS_k$ as the set of all examples that are not close to all $\{\bx_{*,k}\}_{k=1}^K$. Assume $S \neq \{\emptyset\}$ and $|S| > T$. 

The inverse of the covariance matrix after $t - 1$ observations is
\begin{align*}
  \widehat{\bSigma}_t^{-1}
  = \bLambda_t
  = \bI_d + \sum_{\ell = 1}^{t - 1} X_\ell X_\ell\T
  = \sum_{i = 1}^d \lambda_{t, i} \bv_{t, i} \bv_{t, i}\T\,.
\end{align*}
The latter is the eigendecompositon of $\bLambda_t$, where $\lambda_{t, i}$ is the $i$-th largest eigenvalue and $\bv_{t, i}$ is the corresponding eigenvector. Note that $\bLambda_t^{-1} = \sum_{i = 1}^d \lambda_{t, i}^{-1} \bv_{t, i} \bv_{t, i}\T$. To simplify exposition, we assume that all examples have unit length. We analyze the eigenvalues of $\bLambda_t$ first.

\begin{lemma}
\label{lem:eigenvalues} For all $i \in [d]$, $1 \leq \lambda_{t, i} \leq t$. Moreover, let $X_\ell \in S$ hold for all $\ell \in [t - 1]$. Then $\lambda_{t, 1} \geq \alpha_{\min}^2 (t - 1) + 1$.
\end{lemma}
\begin{proof}
The first claim follows directly from the definition of $\bLambda_t$ and that $\normw{X_\ell}{2} \leq 1$. The second claim is proved using the definition of the maximum eigenvalue,
\begin{align*}
  \lambda_{t, 1}
  = \bv_{t, 1}\T \bLambda_t \bv_{t, 1}
  \geq \bx_{*,k}\T \left(\bI_d + \sum_{\ell = 1}^{t - 1} X_\ell X_\ell\T\right) \bx_{*,k}
  \geq \alpha_{\min}^2 (t - 1) + 1\,.
\end{align*}
The last inequality follows from $\bx_{*,k}\T X_\ell \geq \alpha_{\min}$. This completes the proof.
\end{proof}

We continue with claims about the eigenvectors of $\bLambda_t$.

\begin{lemma}
\label{lem:eigenvectors} Let $X_\ell \in S$ hold for all $\ell \in [t - 1]$. Then $\bv_{t, 1} \in S$. Moreover, let $\beta_{\max} \geq 1 - \alpha_{\min}^2$. Then $\bv_{t, i} \in \bS$ for all $i \geq 2$.
\end{lemma}
\begin{proof}
Since all $X_\ell \in S$ and $S$ is a convex set, $\bv_{t, 1} \in S$. Now take any $\bx \in S$ and $i \geq 2$, and note that
\begin{align*}
  (\bx\T \bv_{t, i})^2
  \leq \sum_{i = 2}^d (\bx\T \bv_{t, i})^2
  = 1 - (\bx\T \bv_{t, 1})^2
  \leq 1 - \alpha_{\min}^2\,.
\end{align*}
We use that $\sum_{i = 1}^d (\bx\T \bv_{t, i})^2 = 1$ and $\bx\T \bv_{t, 1} \geq \alpha_{\min}$. Therefore, when $\beta_{\max} \geq 1 - \alpha_{\min}^2$, we have $\bv_{t, i} \in \bS$ for all $i \geq 2$.
\end{proof}

Our analysis has two parts. First, we bound the approximation error under the assumption that $X_\ell \in S$ holds for all $\ell \in [T]$. Second, we show how to choose $\alpha_{\min}$ and $\beta_{\max}$ to guarantee this. We start with the approximation error.

\begin{lemma}
\label{lem:approximation error} Let $X_\ell \in S$ hold for all $\ell \in [T]$. Then for any $\bx_{*,k}$
\begin{align*}
  \bx_{*,k}\T \bLambda_{T + 1}^{-1} \bx_{*,k}
  \leq \frac{1}{\alpha_{\min}^2 T + 1} + (1 - \alpha_{\min}^2)\,.
\end{align*}
\end{lemma}
\begin{proof}
We start with
\begin{align*}
  \bx_{*,k}\T \bLambda_{T + 1}^{-1} \bx_{*,k}
  = \sum_{i = 1}^d \lambda_{T + 1, i}^{-1} \bx_{*,k}\T \bv_{T + 1, i} \bv_{T + 1, i}\T \bx_{*,k}
  \leq \frac{(\bx_{*,k}\T \bv_{T + 1, 1})^2}{\alpha^2_{\min} T + 1} +
  \sum_{i = 2}^d (\bx_{*,k}\T \bv_{T + 1, i})^2\,.
\end{align*}
The inequality uses lower bounds in \cref{lem:eigenvalues}. Then we apply $\bx_{*,k}\T \bv_{T + 1, 1} \leq 1$ and $\sum_{i = 2}^d (\bx_{*,k}\T \bv_{T + 1, i})^2 \leq 1 - \alpha^2_{\min}$.
\end{proof}

Now we prove by induction that $X_\ell \in S$ holds for all $\ell \in [T]$.

\begin{lemma}
\label{lem:greedy example selection} Let $X_\ell \in S$ hold for all $\ell \in [t - 1]$. Suppose that $\displaystyle t \leq \frac{\alpha^2_{\min}}{(\beta + \sqrt{2}) \beta d}$. Then $\bx_t \in S$.
\end{lemma}
\begin{proof}
Our algorithm chooses a example $\bx_t \in S$ when for all $\bx_{*,k}$
\begin{align*}
  \frac{\bx_{*,k}\T \bLambda_t^{-1} \bx \bx\T \bLambda_t^{-1} \bx_{*,k}}
  {1 + \bx\T \bLambda_t^{-1} \bx}
  \geq \frac{\bx_{*,k}\T \bLambda_t^{-1} \by \by\T \bLambda_t^{-1} \bx_{*,k}}
  {1 + \by\T \bLambda_t^{-1} \by}
\end{align*}
holds for any $\bx \in S$ and $\by \in \bS$. Since $0 \leq \bv\T \bLambda_t^{-1} \bv \leq 1$ for any $\normw{\bv}{2} \leq 1$, the above event occurs when
\begin{align*}
  \min_k(\bx_{*,k}\T \bLambda_t^{-1} \bx)^2
  = \min_k\bx_{*,k}\T \bLambda_t^{-1} \bx \bx\T \bLambda_t^{-1} \bx_{*,k}
  \geq 2\max_{k \in [K]} \bx_{*,k}\T \bLambda_t^{-1} \by \by\T \bLambda_t^{-1} \bx_{*,k}
  = 2 \max_{k \in [K]} (\bx_{*,k}\T \bLambda_t^{-1} \by)^2\,.
\end{align*}
We start with an upper bound on the right-hand side,
\begin{align*}
  \max_{k \in [K]}|\bx_{*,k}\T \bLambda_t^{-1} \by|
  \leq \max_{k \in [K]}\sum_{i = 1}^d \lambda_{t, i}^{-1} |\bx_{*,k}\T \bv_{t, i} \bv_{t, i}\T \by|
  \leq \beta_{\max} d\,.
\end{align*}
Here we use $\lambda_{t, i} \geq 1$ (\cref{lem:eigenvalues}), and that $\bv_{t, 1}\T \by \leq \beta_{\max}$ and $\bx_*\T \bv_{t, i} \leq \beta_{\max}$ when $i \geq 2$.

Now we bound the left-hand side as
\begin{align*}
  \min_k|\bx_{*,k}\T \bLambda_t^{-1} \bx|
  \geq \min_k\lambda_{t, 1}^{-1} |\bx_{*,k}\T \bv_{t, 1} \bv_{t, 1}\T \bx| -
  \sum_{i = 2}^d \lambda_{t, i}^{-1} |\bx_{*,k}\T \bv_{t, i} \bv_{t, i}\T \bx|
  &\geq \frac{\alpha^2_{\min}}{t} - \beta^2_{\max} d\,.
\end{align*}
To bound the first term, we use $\lambda_{t, 1} \leq t$, and that $\bx_*\T \bv_{t, 1} \geq \alpha_{\min}$ and $\bv_{t, 1}\T \bx \geq \alpha_{\min}$. To bound the second term, we use $\lambda_{t, i} \geq 1$, and that $\bx_*\T \bv_{t, i} \leq \beta_{\max}$ and $\bv_{t, i}\T \bx \leq \beta_{\max}$.

Now we chain all inequalities and get that our algorithm chooses a example $\bx_t \in S$ when
\begin{align*}
  t
  \leq \frac{\alpha^2_{\min}}{(\beta_{\max} + \sqrt{2}) \beta_{\max} d}\,.
\end{align*}
This completes the proof.
\end{proof}

\subsection{Proof of \cref{thm:sal}}
\label{sec:sal proof}

Consider the test examples $\{\bx_*\}_{k=1}^K$. Recall that $H_t=\left(X_\ell, Y_\ell\right)_{\ell \in[t-1]}$ is the history till round $t$. Then the posterior variance is
\begin{align*}
  \wSigma_t
  = \left(\bSigma_0^{-1} + \sigma^{-2} \sum_{\ell = 1}^{t - 1} X_\ell X_\ell\T\right)^{-1}
\end{align*}
and $\wtheta_t=\wSigma_t\left(\bSigma_0^{-1} \btheta_0+\sigma^{-2} \sum_{\ell=1}^{t-1} X_\ell Y_\ell\right)$. Now fix a $X_t$ at round $t$. Then $Y_t=X_t^{\top}\btheta_* + \epsilon_t$ where $\btheta_* \mid H_t \sim \N\left(\wtheta_{t}, \wSigma_t\right)$ and $\epsilon_t \sim \N\left(0, \sigma^2\right)$. Then $\btheta_* |H_{t+1} \sim \N\left(\wtheta_{t+1}, \wSigma_{t+1}\right)$ holds for any $Y_t$.

Now fix an $\bx_i\in\cXt$ and add it to $\wSigma_{t}$ such that
\begin{align*}
  \wSigma_{t,i}
  = \left(\bSigma_0^{-1} + \sigma^{-2} \left(\sum_{\ell = 1}^{t - 1} X_\ell X_\ell\T + \bx_i\bx_i^\top\right)\right)^{-1}.
\end{align*}


Define $\hat{\sigma}_{t, i, k}^2 = \frac{1}{m} \sum_{j = 1}^m (\tilde{Y}^{(j,1)}_{t, i, k} - \tilde{Y}^{(j,2)}_{t, i, k})^2$ and the $\E[\hat{\sigma}_{t, i, k}^2] = \sigma_{t, i, k}^2$.
Then define $\sigma_{t,i,\max}^2 = \max_{k\in [K]} \E[\hat{\sigma}_{t, i, k}^2]$.
Then using \Cref{lemma:conc} we can show that with probability $(1-\delta)$
\begin{align}
\sigma_{t,i,\max}^2\left[1 - 2 \sqrt{\frac{\log (1 / \delta)}{m}} \right] &\leq \max_{k \in [K]}\hat{\sigma}_{t, i, k}^2 \leq \sigma_{t, i, \max}^2\left[1 + 2 \sqrt{\frac{\log (1 / \delta)}{m}} + \frac{2\log (1 / \delta)}{m} \right] 
\label{eq:conc-variance}
\end{align}
Set $m \geq 8\log(1/\delta)$ in \eqref{eq:conc-variance}. It follows then that
\begin{align*}
  \frac{1}{2}\sigma^2_{t, i,\max} \leq \max_{k \in [K]}\frac{1}{m}& \sum_{j = 1}^m(\widetilde{Y}^{(j,1)}_{t, i, k} - \widetilde{Y}^{(j,2)}_{t, i, k})^2 \leq \frac{5}{2}\sigma^2_{t, i,\max}.
\end{align*}



Hence, to minimize the quantity $\max_{k\in[K]}\widetilde{Y}^{(j,1)}_{t,i,k} - \widetilde{Y}^{(j,2)}_{t,i,k}$ we should be minimizing the variance $2\max_{k\in [K]}\bx_{*,k}^{\top} \wSigma_{t,i} \bx_{*,k} + \sigma^2$. 
Observe that minimizing the variance in \sal\ leads to minimizing the quantity $2\max_{k\in[K]}\bx_{*,k}^{\top} \wSigma_{t,i} \bx_{*,k} + \sigma^2$ which is same as minimizing the score  $\max_{k\in[K]}\bx_{*,k}^{\top} \wSigma_{t,i} \bx_{*,k}$ for \go. The claim of the theorem follows.

\begin{lemma}
\label{lemma:conc}
Fix round $t \in [T]$, a sample example $\bx_i$, and a test example $\bx_{*,k}$, and failure probability $\delta \in(0,1)$. 
Suppose that $m>4 \log (1 / \delta)$. Define $\hat{\sigma}_{t, i, k}^2 = \frac{1}{m} \sum_{j = 1}^m (\tilde{Y}^{(j,1)}_{t, i, k} - \tilde{Y}^{(j,2)}_{t, i, k})^2$ and the $\E[\hat{\sigma}_{t, i, k}^2] = \sigma_{i,k}^2$. 
Then
$$
\Pb\left(\hat{\sigma}_{t, i, k}^2 \leq \sigma_{i,k}^2\left[1-2 \sqrt{\frac{\log (1 / \delta)}{m}} \right]\right) \leq \delta
$$
holds with probability at least $1-\delta$. Analogously,
$$
\Pb\left(\hat{\sigma}_{t, i, k}^2 \geq \sigma_{i,k}^2\left[1+2 \sqrt{\frac{\log (1 / \delta)}{m}}+\frac{2 \log (1 / \delta)}{m}\right]\right) \leq \delta
$$
holds with probability at least $1-\delta$.
\end{lemma}

\begin{proof}
    Fix an $\bx_i\in\cXt$ and add it to $\wSigma_{t}$. Denote this new co-variance matrix as $\wSigma_{t, i}$ such that
    \begin{align*}
      \wSigma_{t,i}
      = \left(\bSigma_0^{-1} + \sigma^{-2} \left(\sum_{\ell = 1}^{t - 1} X_\ell X_\ell\T + \bx_i\bx_i^\top\right)\right)^{-1}.
    \end{align*}
    Let $\widetilde{Y}^{(1)}_{t, i, j}=\bx_{*, k}^{\top} \btheta_* + \epsilon_{t,i,j,1}$, where $\btheta_* \mid H_{t} \sim \N\left(\wtheta_{t}, \wSigma_{t+1}\right)$ and $\epsilon_{t,i,j,1} \sim \N\left(0, \sigma^2\right)$.  This yields that $\widetilde{Y}^{(1)}_{t, i, j} \sim \N\left(\bx_{*,k}^{\top} \wtheta_{t+1}, \bx_{*,k}^{\top} \wSigma_{t,i} \bx_{*,k} + \sigma^2\right)$.
Similarly $\widetilde{Y}^{(2)}_{t, i, j}=\bx_{*, k}^{\top} \btheta_* + \epsilon_{t, i, j, 2}$, where $\btheta_* \mid H_{t} \sim \N\left(\wtheta_{t+1}, \wSigma_{t,i}\right)$ and $\epsilon_{t,i,j,2} \sim \N\left(0, \sigma^2\right)$. Therefore, we get that
\begin{align*}
    \widetilde{Y}^{(j,1)}_{t,i,k} - \widetilde{Y}^{(j,2)}_{t,i,k} \sim \N(0, 2\bx_{*,k}^{\top} \wSigma_{t,i} \bx_{*,k} + \sigma^2).
\end{align*}
Now we proceed the same way as in Lemma 2 of \citet{lalitha2023fixed}. Using  Cochran's theorem, we have that $\hat{\sigma}_{t, i, k}^2 m / \sigma_{i,k}^2$ is a $\chi^2$ random variable with $m$ degrees of freedom. Then using (4.4) and Lemma 1 of \citet{laurent2000adaptive} we can show that
\begin{align}
\Pb\left(m-\frac{\hat{\sigma}_{t, i, k}^2 m}{\sigma_{i,k}^2} \geq 2 \sqrt{m \log (1 / \delta)}\right) \leq \delta \label{eq:prob-bound}
\end{align}
Dividing both sides of \eqref{eq:prob-bound} in the probability by $m$, and multiplying by $\sigma_{i,k}^2$, we can get the following
$$
\Pb\left(\sigma_{i,k}^2(1-2 \sqrt{\log (1 / \delta) / m}) \geq \hat{\sigma}_{t, i, k}^2\right) \leq \delta .
$$
Observe that $1-2 \sqrt{\log (1 / \delta) / m}>0$, we can divide both sides by it and get the first claim of the lemma.
The second claim is proved by (4.3) in \citet{laurent2000adaptive}, an immediate corollary of their Lemma 1, we have
$$
\Pb\left(\frac{\hat{\sigma}_{t, i, k}^2 m}{\sigma_{i,k}^2}-m \geq 2 \sqrt{m \log (1 / \delta)}+2 \log (1 / \delta)\right) \leq \delta
$$
The claim of the lemma follows.
\end{proof}

\section{Additional Experiments and Results}
\label{app:addl_expt}
We use NVIDIA GeForce RTX 3090 GPU with 24GB RAM to load the Large Language Models for inference. The Mistral-7B  model requires less than 16GB RAM, and Vinuna-13B model requires less than 22 GB RAM during execution.
To run Falcon-40B model we use AWS ml-g5.12xlarge machine. To run the full set of experiments it takes 24-27 hours of compute job.
We now briefly discuss the various datasets used in this work.

\subsection{Datasets}
\label{sec:datasets}

We now briefly describe the datasets used for our experiments. All the real-life datasets are from UCI \citep{uci} and OpenML \citep{OpenML2013} repositories. We use $4$ classification and $3$ regression datasets from UCI and OpenML. Additionally, we use $2$ custom datasets for movie names and entity names for classification task in our experiments. These are as follows:

\noindent
\textbf{(1)} \textit{Iris}: We use this UCI dataset for classification task. This dataset consists of four features of flowers and three classes of flowers. We use all four features in the prompts as well as estimating the score for selecting the next action. The dataset consists of $150$ instances. We randomly choose $K=20$ as test examples and the remaining instances as training examples.

\noindent
\textbf{(2)} \textit{Banknote-authentication}: We use this OpenML dataset for classification task. This dataset consists of five features of banknotes and two classes for identifying fake or original banknote. Out of these five features, we use four features in the prompts as well as estimating the score for selecting the next action. The dataset consists of $150$ instances. We randomly choose $K=20$ as test examples and the remaining instances as training examples.

\noindent
\textbf{(3)} \textit{Balance-scale}: We use this OpenML dataset for classification task. This dataset consists of five features of a scale and three classes of whether the scale tips left/right or is balanced. Out of these five features, we use four features in the prompts as well as estimating the score for selecting the next action. The dataset consists of $625$ instances. We randomly choose $K=20$ as test examples and the remaining instances as training examples.

\noindent
\textbf{(4)} \textit{Thyroid-new}: We use this OpenML dataset for classification task. This dataset consists of six features for thyroids and three classes. Out of these six features, we use five features in the prompts as well as estimating the score for selecting the next action. The dataset consists of $215$ instances. We randomly choose $K=20$ as test examples and the remaining instances as training examples.

\noindent
\textbf{(5)} \textit{Movie-name}: We use this custom dataset for classification task. This dataset consists of movie names across five genres (classes) romance, horror, thriller, sport, and action. We convert the movie names into $768$ dimensional feature embeddings using Instructor embeddings. Note that in the prompt to the \LLM\ we only pass the movie names and the goal is to identify the common genre.  The dataset consists of $100$ instances. We randomly choose $K=20$ as test examples and the remaining instances as training examples.

\noindent
\textbf{(6)} \textit{Movie-theme}: We use this custom dataset for classification tasks in identifying a common theme between pairs of movies. This dataset consists of movie names across five themes (classes) good-vs-evil, man-vs-nature, redemption, Love conquers all, and coming-of-age. We convert the movie names into $768$ dimensional feature embeddings using Instructor embeddings. Note that in the prompt to the \LLM\ we only pass the pair of movie names and the goal is to identify the common theme. The dataset consists of $100$ instances. We randomly choose $K=20$ as test examples and the remaining instances as training examples.

\noindent
\textbf{(7)} \textit{Entity-name}: We use this custom dataset for classification task. This dataset consists of entity names across five entity types (classes) like mountains, seas, rivers, vehicles, and celebrities. Again, we convert the entity names into $768$ dimensional feature embeddings using Instructor embeddings. Note that in the prompt to the \LLM\ we only pass the entity names and the goal is to identify the entity type. The dataset consists of $100$ pairs of instances. We randomly choose $K=20$ pairs as test examples and the remaining instances as training examples. 

\noindent
\textbf{(8)} \textit{Fifa}: We use this OpenML dataset for the regression task. This dataset consists of six features of players and the clubs they joined as targets. Out of these six features, we use five features in the prompts as well as estimating the score for selecting the next action.  The dataset consists of $18063$ instances. We randomly choose $K=20$ test examples and another $200$ examples as training examples.


\noindent
\textbf{(9)} \textit{Machine-cpu}: We use this OpenML dataset for regression tasks. This dataset consists of seven features of machine cpu and the target variable is the performance of the cpu. Out of these seven features, we use five features in the prompts as well as estimating the score for selecting the next action. The dataset consists of $209$ instances. We randomly choose $K=20$ test examples and the remaining examples as training examples.

\subsubsection{ARC Experiment}
\label{app:arc}
In the recent works of \citet{mirchandani2023large, srivastava2022beyond} they showed that \LLMs\ behave as general pattern-matching machine. In fact they showed that \LLMs\ can be used to solve tasks from Abstract Reasoning Corpus (ARC) tasks. In the following experiments, we choose two such tasks: (1) ARC expansion and contraction experiment and (2) ARC rotation experiment.

\textbf{(1) ARC expansion and contraction experiment:} In the expansion and contraction experiment, there are two sets of matrices of dimension $4\times 4$ which constitute half the examples of the feature space $\X$. The first set of input matrices have integer values in their center $2\times 2$ cells while all the other cells are $0$. The label space $\Y$ of this $4\times 4$ matrix is also a $4\times 4$ matrix where the $4$ inner cells have moved to the $4$ corners. These matrices are termed as expansion matrices. 

Similarly, the other set of $4\times 4$ matrices have the $4$ non-zeros values in their corners. These constitute the remaining examples in $\X$. Then the label space $\Y$ is given by $4\times 4$ matrix where the four non-zeros cells come to the center and all the other cell values are $0$. These matrices are termed as contraction matrices. This is shown in \Cref{fig:arc-env1}. 

Therefore, the feature space $\X$ consists of both the expansion and contraction matrices. At every trial, and $n$ training examples and $K$ test examples are chosen randomly from $\X$. Then we run all baselines for $T$ iterations where the classification accuracy is calculated if the \LLM\ is able to predict the exact matching. This experiment is shown in \Cref{tab:arc-pcfg}.

\textbf{(2) ARC rotation experiment:} In the rotation experiment, there are again two sets of matrices of dimension $4\times 4$ which constitute half the examples of the feature space $\X$. The first set of matrices are have integer values in their four corner cells while all the other cells are $0$. The label space $\Y$ of this $4\times 4$ matrix is also a $4\times 4$ matrix where the $4$ corner cell values have moved $90^{\circ}$ in the clockwise direction. These matrices are termed as clockwise matrices. 

Similarly, the other set of $4\times 4$ matrices have the $4$ non-zeros values in their corners. These constitute the remaining examples in the feature space $\X$. Then the label space $\Y$ is given by $4\times 4$ matrix where the four non-zeros cells have moved $15^{\circ}$ in the anti-clockwise direction and all the other cell values are $0$. These matrices are termed as anti-clockwise matrices. This is shown in \Cref{fig:arc-env2}. 

Therefore, the feature space $\X$ consists of both the clockwise and anti-clockwise matrices. At every trial, $n$ training examples and $K$ test examples are chosen randomly. Then we run all the baselines for $T$ iterations where the classification accuracy is calculated if the \LLM\ is able to predict the exact matching. This experiment is shown in \Cref{tab:arc-pcfg}.

\begin{figure}[!ht]
\centering
\begin{tabular}{cc}
\subfigure[Classification on ARC Expansion and Contraction Experiment]{\includegraphics[scale = 0.26]{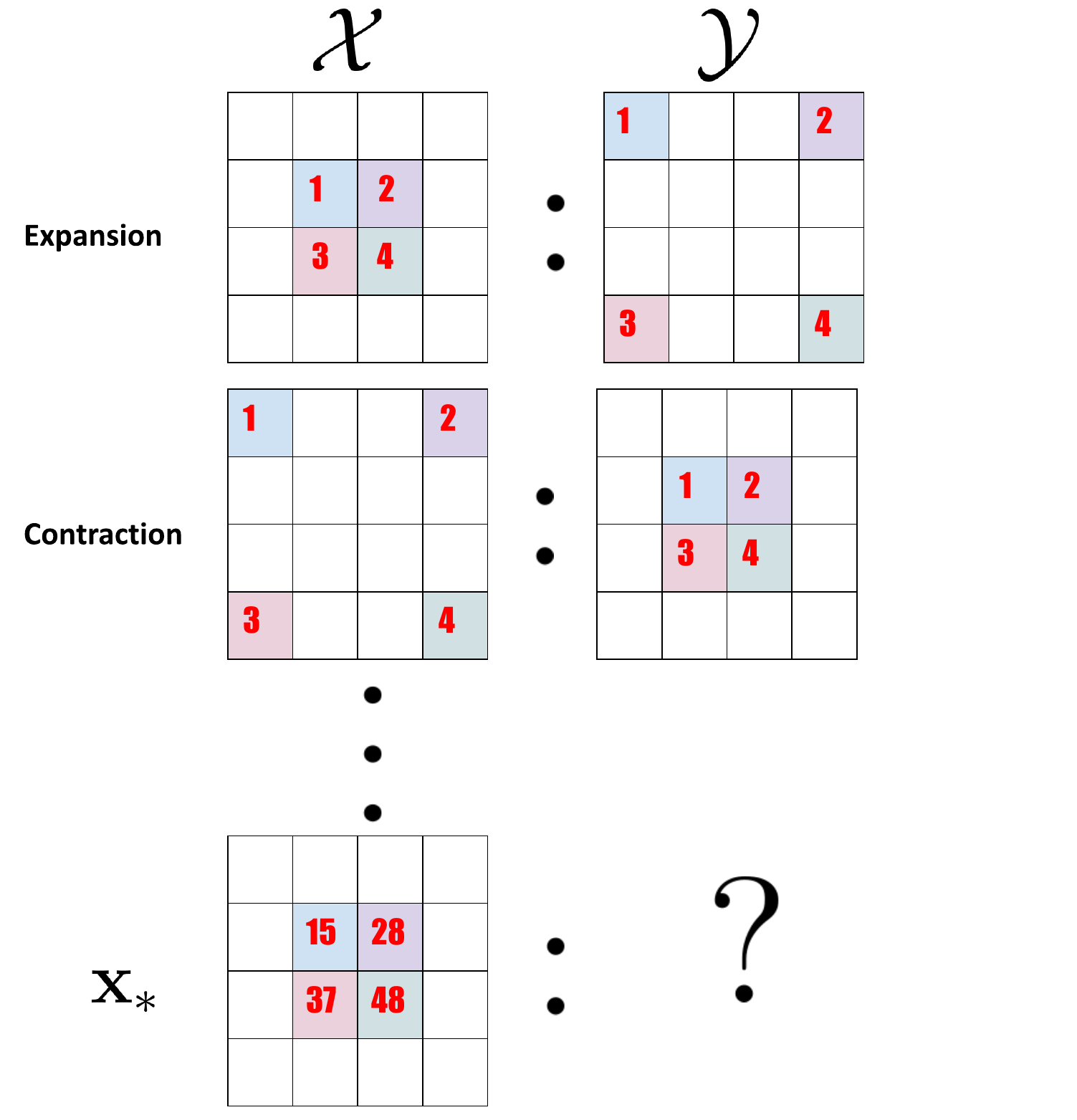}\label{fig:arc-env1}} &
\subfigure[Classification on ARC Rotation Experiment]{\includegraphics[scale = 0.24]{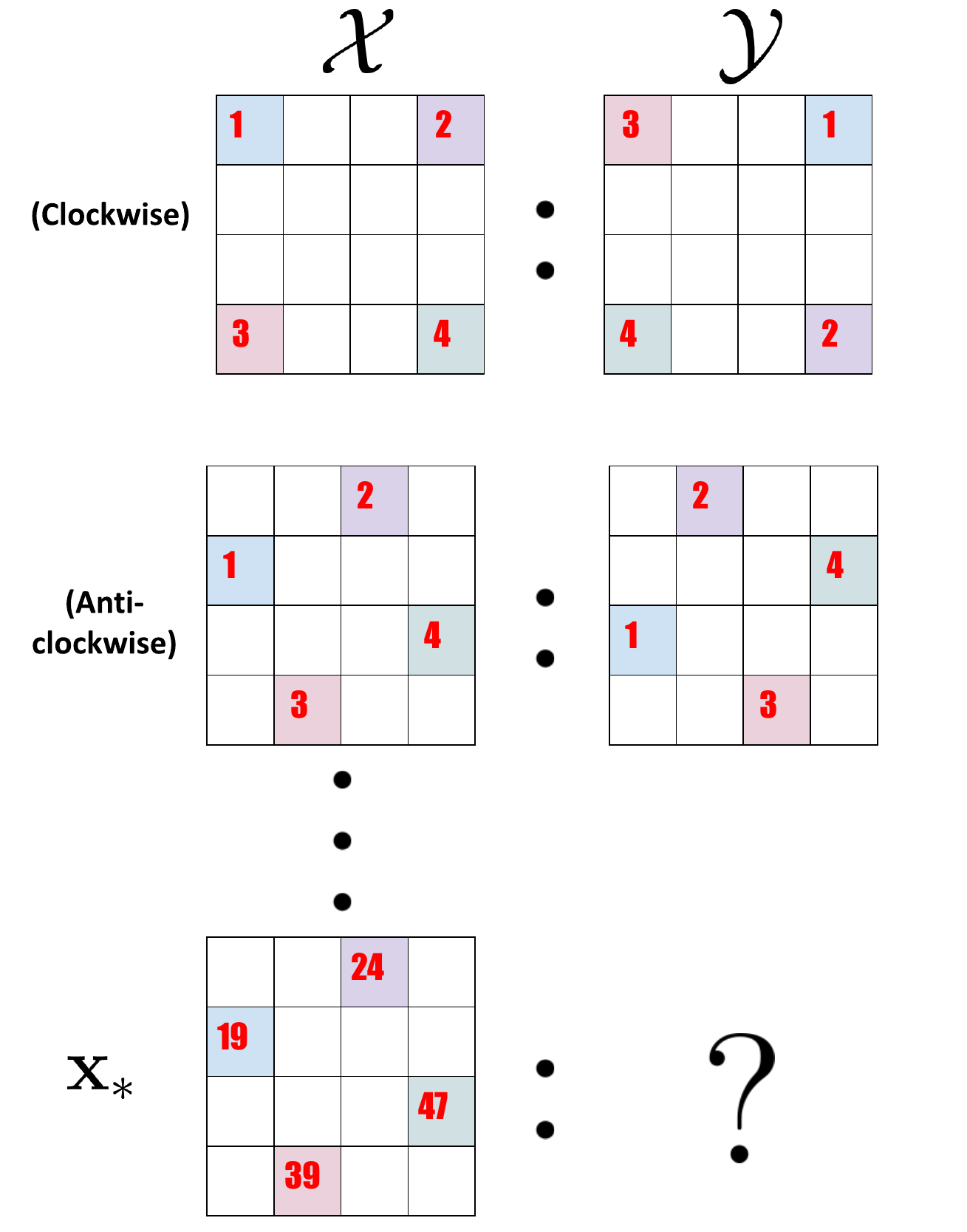}\label{fig:arc-env2}}  \\
\end{tabular}
\caption{Explanation of ARC tasks} 
\label{fig:expt-arc}
\end{figure}

\subsubsection{PCFG Experiment}
\label{app:pcfg}
In this experiment the goal is to predict the next output of a sequence. In the following experiments we choose two such tasks: (1) PCFG add-subtract experiment and (2) PCFG repeat experiment.

\textbf{(1) PCFG add-subtract experiment:} In the add-subtract experiment, there are two sets of sequence of $4$ integers. The first set of sequence of $4$ integers consists of odd integer values which constitute half the examples in the feature space $\X$. The label space $\Y$ of this sequence of $4$ odd examples is sequences of $5$ integers where the last integer is padded to the original sequence by adding one to the last odd integer. These sequences are termed as add examples. 

Similarly, the other set of examples consists of a sequence of $4$ even integer values which constitute the remaining examples in the feature space $\X$. The label space $\Y$ of this sequence of $4$ even integer examples is a sequence of $5$ integers where the last integer is padded to the original sequence by subtracting one from the last even integer. These examples are termed as even examples. This is shown in \Cref{fig:pcfg-env1}.

Therefore, the feature space $\X$ consists of both the odd and even sequence of $4$ integer value examples. At every trial, $n$ training and $K$ test examples are chosen randomly. Then we run all the baselines for $T$ iterations where the classification accuracy is calculated if the \LLM\ is able to predict the exact matching. This experiment is shown in \Cref{tab:arc-pcfg}.

\textbf{(2) PCFG repeat experiment:} In the repeat experiment, there are two sets of sequence of $4$ integers. The first set of sequence of $4$ odd integer values constitute half the examples of the feature space $\X$. The label space $\Y$ of this sequence of $4$ integers examples is a sequence of $5$ integers where the last integer is padded to the original sequence by repeating the first odd integer. These sequences are termed as odd-repeat examples. 

Similarly, the other set of examples consists of sequence of $4$ even integer values which constitute the remaining examples in the feature space $\X$. The label space $\Y$ of these sequence of $4$ integer value examples is a sequence of $5$ integers where the last integer is padded to the original sequence by repeating the second even integer. These examples are termed as even-repeat examples. This is shown in \Cref{fig:pcfg-env2}.

Therefore, the feature space $\X$ consist of both the odd-repeat and even-repeat examples. At every trial, $n$ training examples and $K$ test examples are chosen randomly. Then we run all the baselines for $T$ iterations where the classification accuracy is calculated if the \LLM\ is able to predict the exact matching. This experiment is shown in \Cref{tab:arc-pcfg}.

\begin{figure}[!ht]
\centering
\begin{tabular}{cc}
\subfigure[Classification on PCFG add-subtract Experiment]{\includegraphics[scale = 0.2]{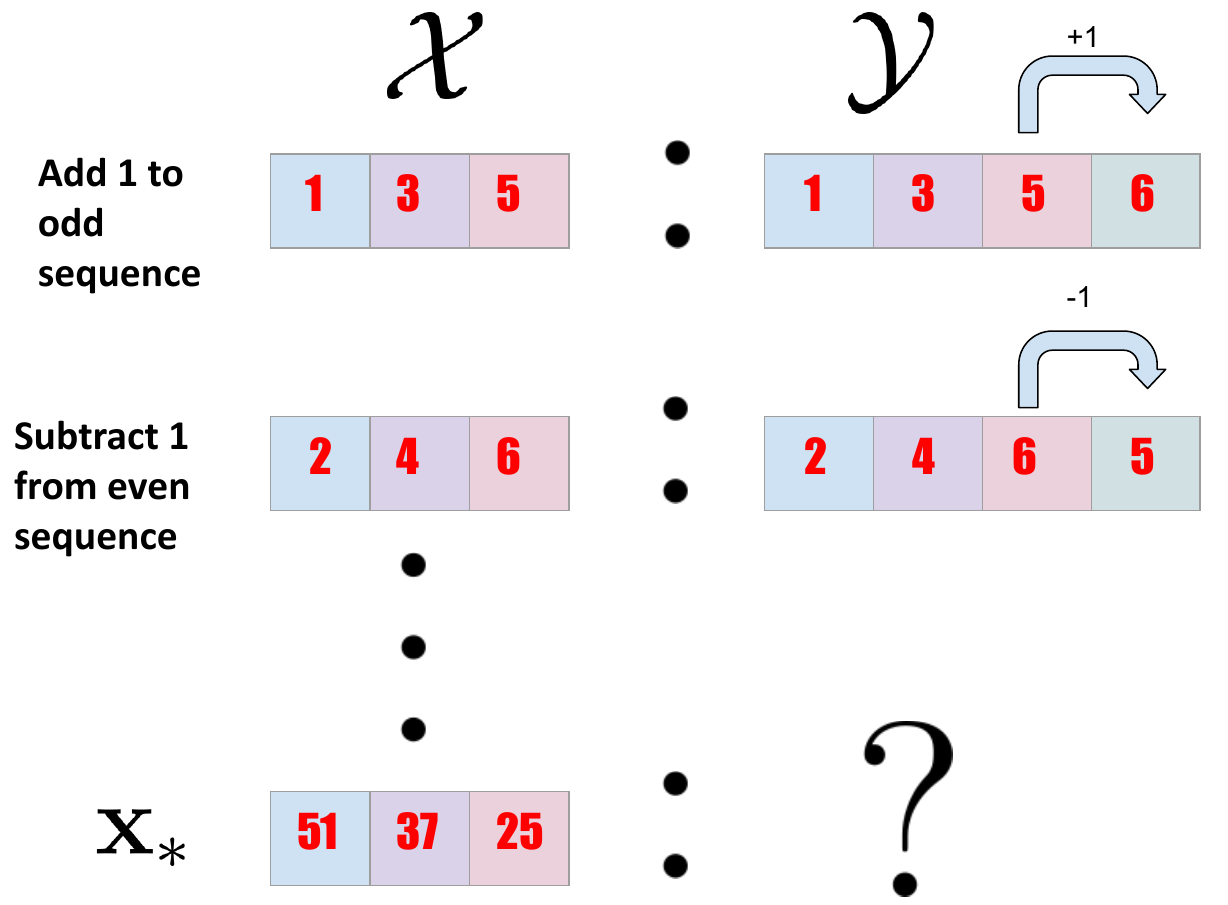}\label{fig:pcfg-env1}} &
\subfigure[Classification on PCFG repeat Experiment]{\includegraphics[scale = 0.2]{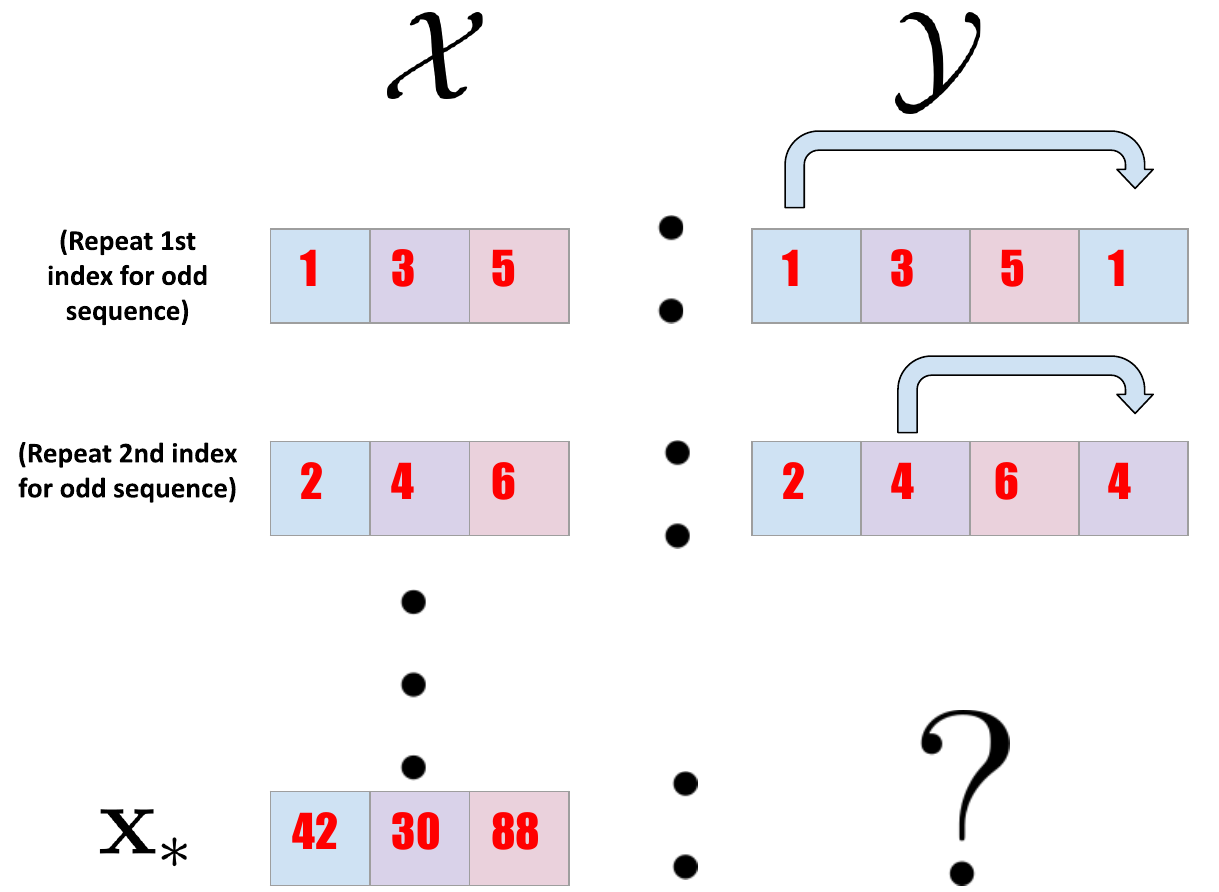}\label{fig:pcfg-env2}}  \\
\end{tabular}
\caption{Explanation of PCFG task.}
\label{fig:expt-pcfg}
\end{figure}

\section{Prompt Examples}
\label{app:prompt-ex}
\textbf{Classification Dataset Prompts :} Below we give an example of how we use the prompts to be used in the \LLM\ for the Iris misclassification task. Similar types of prompts can be found in \citet{dinh2022lift, suzgun2022challenging}. This is shown in \Cref{fig:class-prompt}. Note that since we have the feature representation of the training and test examples from the dataset, we directly use them as $\bx_i$ and $\bx_{*,k}$.

\textbf{Regression Dataset Prompts :} In \Cref{fig:reg-prompt} we give an example of a prompt for regression task in Fifa dataset. Note that since we have the feature representation of the training and test examples from the dataset, we directly use them as $\bx_i$ and $\bx_{*,k}$.


\begin{figure}[!hbt]
\centering
\begin{tabular}{cc}
\subfigure[Classification Prompt]{\hspace*{-1.5em}\includegraphics[scale = 0.23]{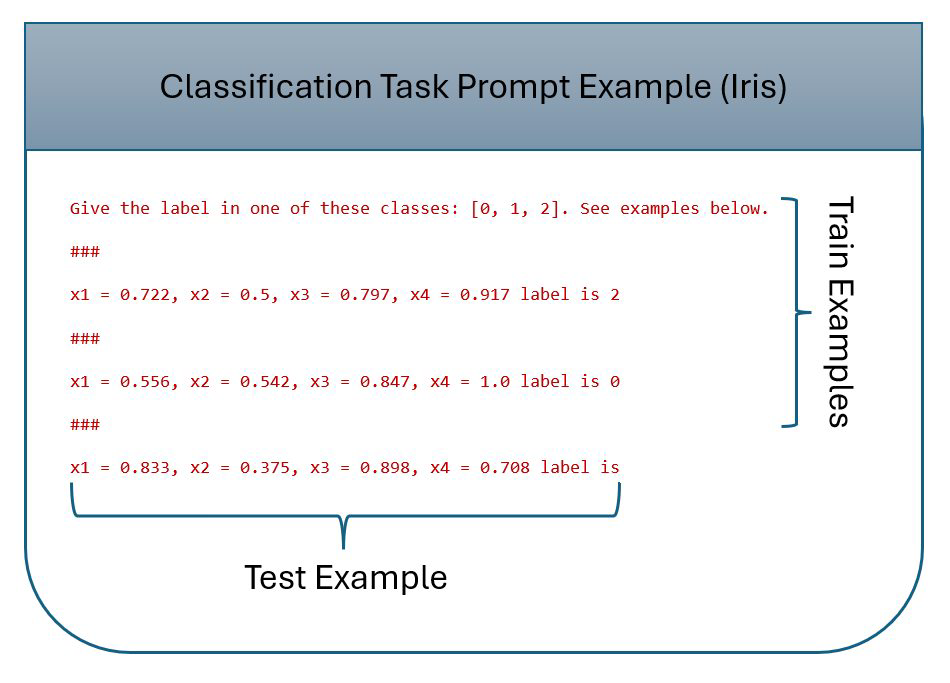}\label{fig:class-prompt}} &
\hspace*{-1.5em}\subfigure[Regression Prompt]{\includegraphics[scale = 0.23]{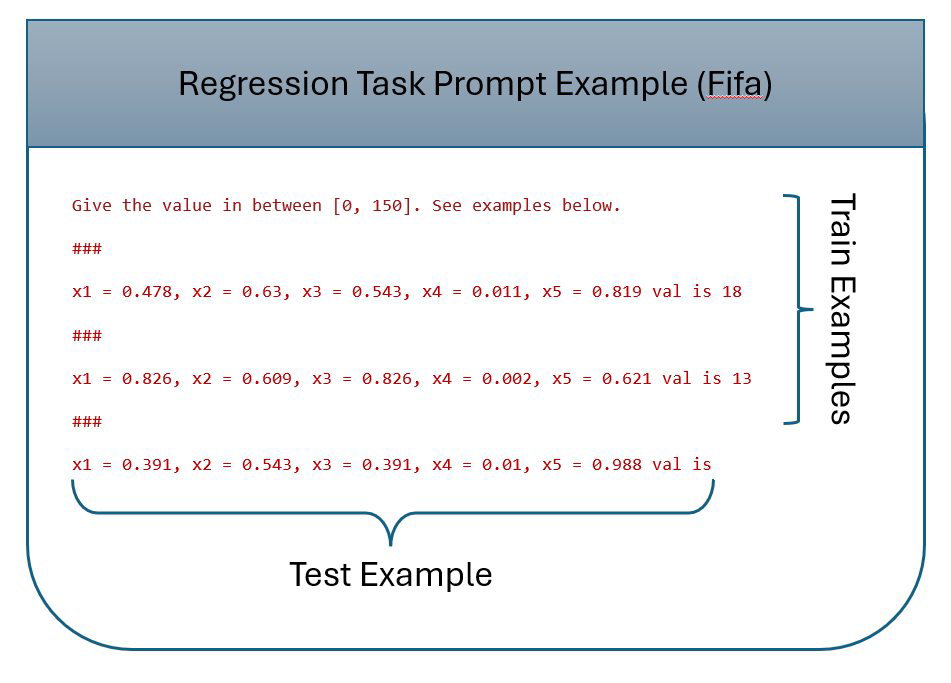} \label{fig:reg-prompt}} \\
\subfigure[Movie Prompt]{\hspace*{-1.5em}\includegraphics[scale = 0.23]{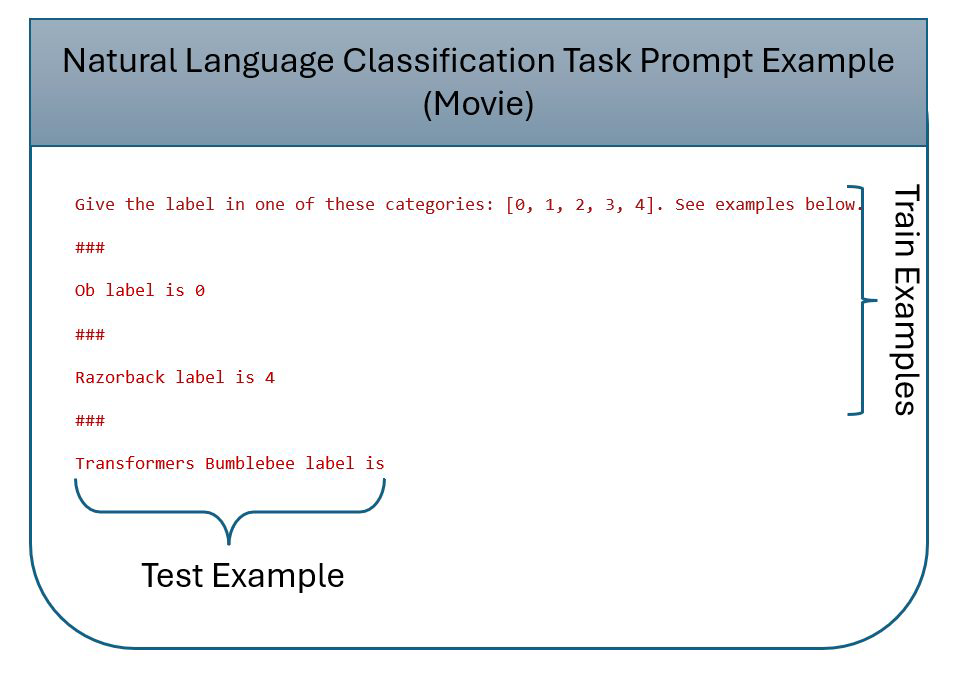}\label{fig:movie-prompt}} &
\hspace*{-1.5em}\subfigure[Entity Prompt]{\includegraphics[scale = 0.23]{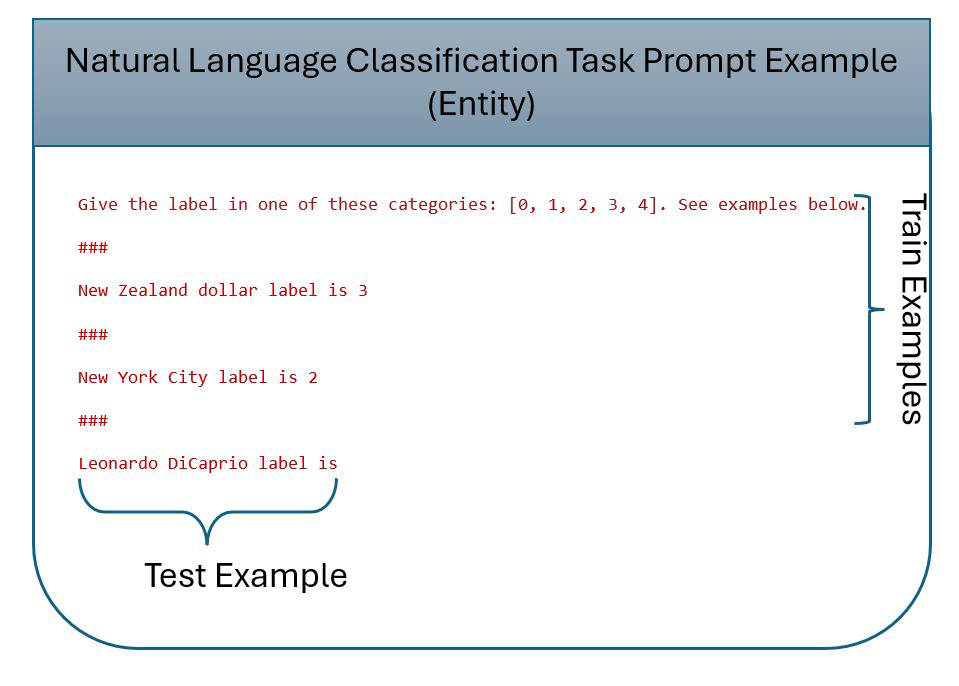} \label{fig:entity-prompt}} 
\end{tabular}
\vspace*{-1em}
\caption{Prompt examples for Classification, Regression, Movie, and Prompt}
\label{fig:expt-1}
\vspace{-0.7em}
\end{figure}

\begin{figure}[!hbt]
\centering
\begin{tabular}{cc}
\hspace*{-1.5em}\subfigure[Theme Prompt]{\includegraphics[scale = 0.23]{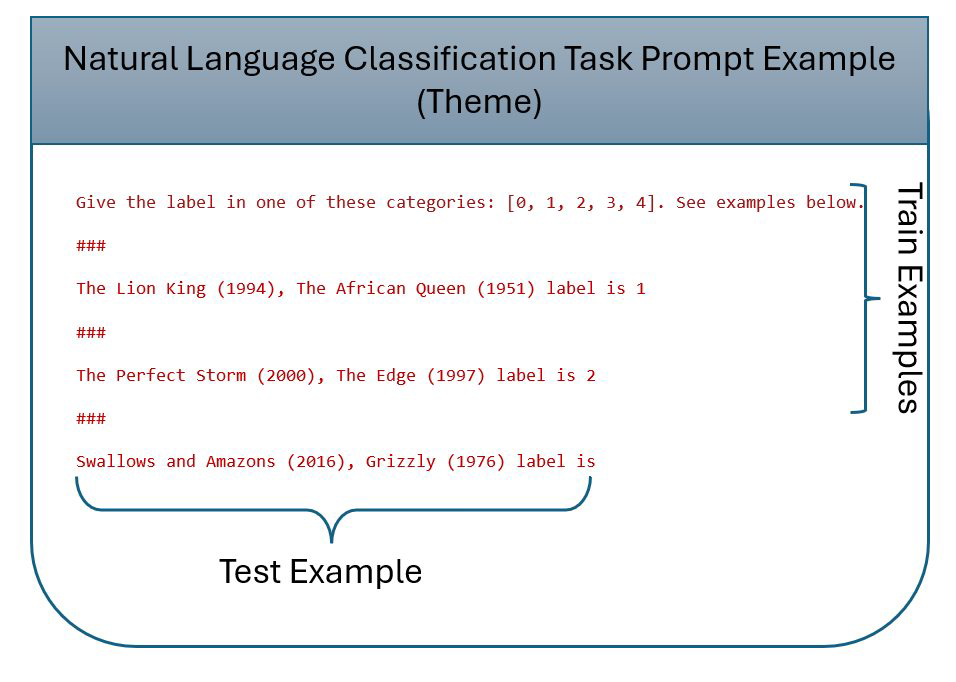} \label{fig:theme-prompt}} &
\hspace*{-1.5em}\subfigure[PCFG Prompt]{\includegraphics[scale = 0.23]{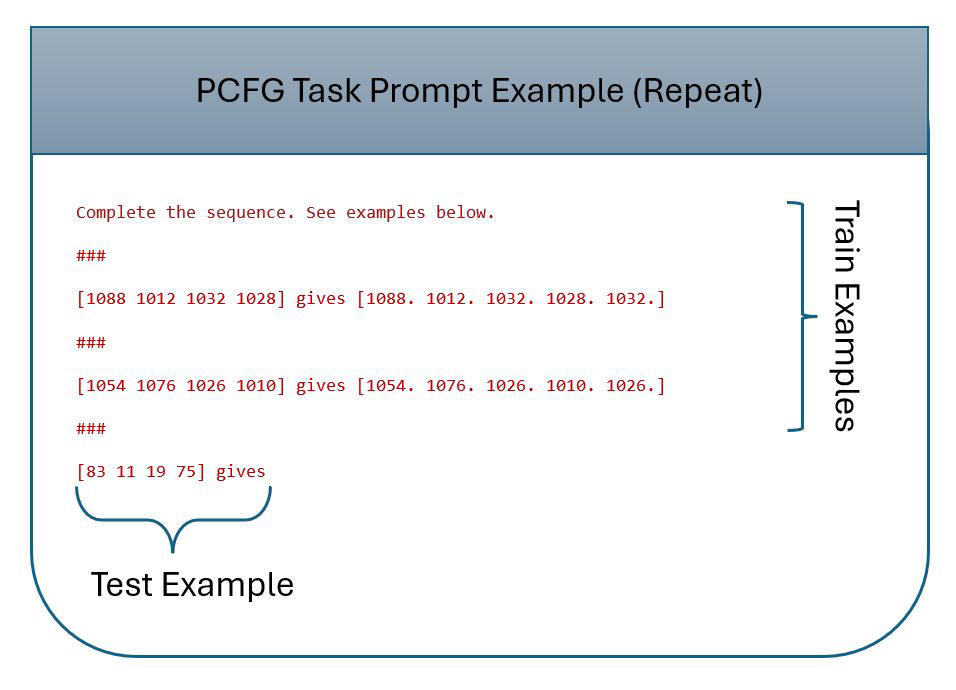} \label{fig:pcfg-prompt}} 
\end{tabular}
\vspace*{-1em}
\caption{Prompt examples for Theme and PCFG tasks}
\label{fig:expt-2}
\vspace{-0.7em}
\end{figure}
\begin{figure}[!hbt]
\centering
\begin{tabular}{c}
\hspace*{-1.5em}\subfigure[ARC Prompt]{\includegraphics[scale = 0.27]{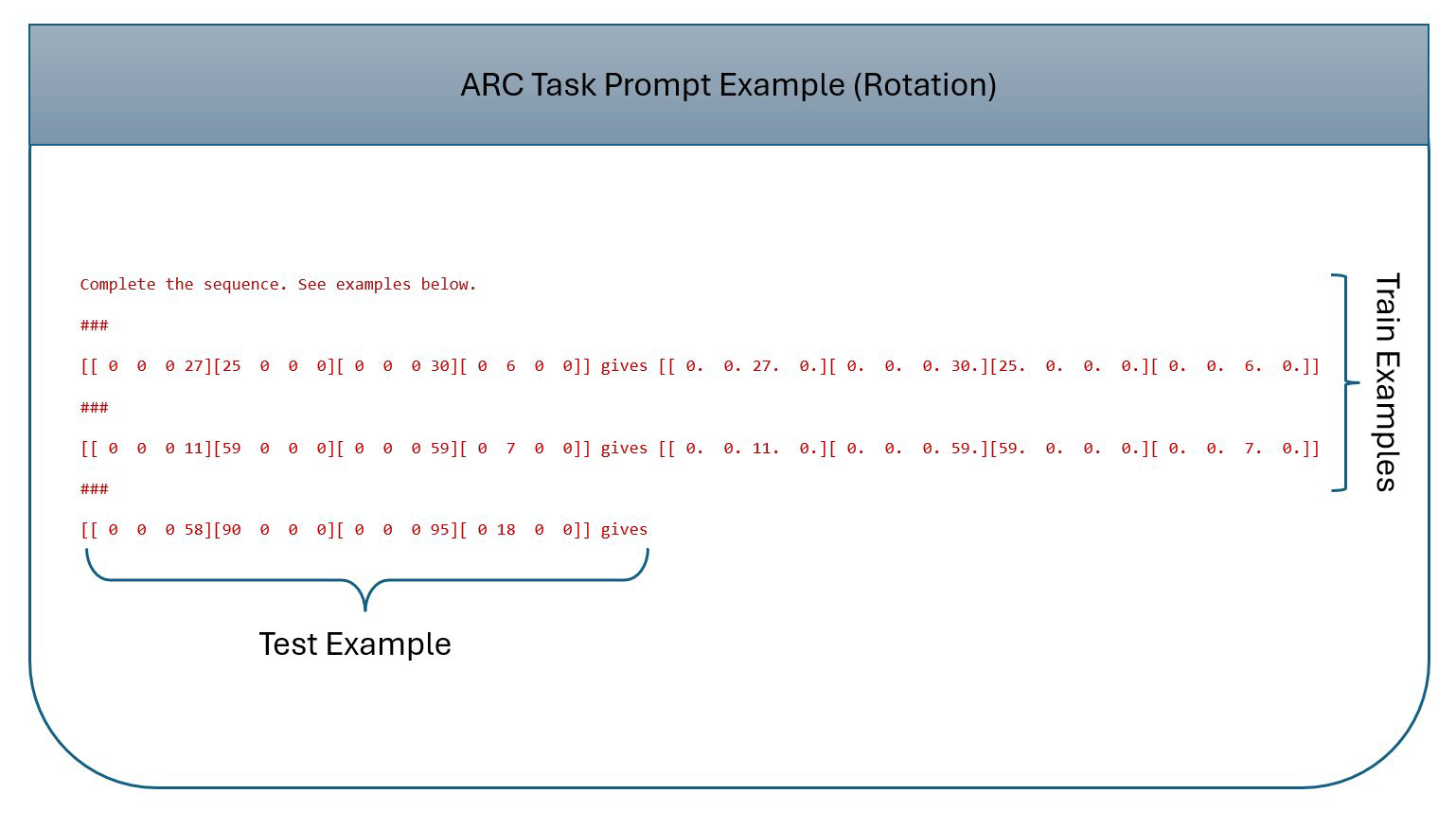} \label{fig:arc-prompt}}
\end{tabular}
\vspace*{-1em}
\caption{Prompt examples for ARC task}
\label{fig:expt-3}
\vspace{-0.7em}
\end{figure}

\textbf{Movie Theme Experiment:} We use a similar technique as in Iris dataset for this setting. The labels of the pairs of movies belong to $5$ classes as follows: good-vs-evil, man-vs-nature, redemption, Love conquers all, and coming-of-age. At every iteration, we pass $K$ pairs of movie test examples where each $\bx_{*,k}$ is a pair of movies. In the example below we have $\bx_{*,1} = \text{['Swallows and Amazon (2016), Grizzly (1976)']}$. 
Note that we feed the natural language text to the \LLM\ as prompts as shown in \Cref{fig:theme-prompt}. 
However, to run \go, \sal, and other baselines we require a featurization of these natural language prompts. We obtain a $768$ dimensional featurized representation of the pairs of movies 'Monsters Inc, Frozen (2013)' using Instructor embedding \citep{INSTRUCTOR}. This constitutes $\bx_i\in\R^{768}$ and $\bx_{*,k}\in\R^{768}$.







\textbf{Movie Name Experiment:} We use a similar technique as in Iris dataset for this setting. The labels of the movie genres belong to $5$ classes as follows: romance, horror, thriller, sport, and action. At every iteration we pass a set of test movie name examples where each $\bx_{*,k}$ is now movie name.  Note that we feed the natural language text to the \LLM\ as prompts as shown in \Cref{fig:movie-prompt}. 
However, to run \go, \sal, and other baselines we require a featurization of these natural language prompts. We obtain a $768$ dimensional featurized representation of the movie names using Instructor embedding \citep{INSTRUCTOR}. This constitutes $\bx_i\in\R^{768}$ and $\bx_{*,k}\in\R^{768}$.









\textbf{Entity Name Experiment:} The labels of the entity genres belong to $5$ classes as follows: mountains, seas, rivers, vehicles, and celebrities. At every iteration, we pass a set of test entity name examples where each $\bx_{*,k}$ is now an entity name.  Note that we feed the natural language text to the \LLM\ as prompts as shown in \Cref{fig:entity-prompt}. 
However, to run \go, \sal, and other baselines we require a featurization of these natural language prompts. Again, we obtain a $768$ dimensional featurized representation of the entity names using Instructor embedding \citep{INSTRUCTOR}. This constitutes $\bx_i\in\R^{768}$ and $\bx_{*,k}\in\R^{768}$.

\textbf{PCFG Experiment:} We show an example of this prompt in \Cref{fig:pcfg-prompt}. Here we concatenate the sequence to obtain training examples $\bx_i$ and test examples $\bx_{*,k}$. So a sequence of $4$ integers of length $4$ will be represented by $\bx_i, \bx_{*,k} \in \R^{16}$. Similarly the label $Y_i$ and $Y_{*,k}$ consist of sequence of $5$ integers of length $4$ which we concatentate to get a vector of length $\R^{20}$.

\textbf{ARC Experiment:} We show an example of this prompt in \Cref{fig:arc-prompt}. Here we vectorized the $4 \times 4$ matrix to obtain training examples $\bx_i\in \R^{16}$ and test examples $\bx_{*,k}\in \R^{16}$. Similarly the label $Y_i$ and $Y_{*,k}$ consist of  vectorized matrices of length $\R^{16}$.



\section{Table of Notations}
\label{table-notations}

\begin{table}[!tbh]
    \centering
    \begin{tabular}{|p{20em}|p{21em}|}
        \hline\textbf{Notations} & \textbf{Definition} \\\hline
        $n$ & Total unlabeled examples \\\hline
        $d$ & Dimension of the feature \\\hline
        $\cX$  & Feature set \\\hline
        $\cY$  & Label space\\\hline
        $\btheta_{*}$  & Unknown model parameter\\\hline
        $\bx_i$  & Feature of sample example $i$\\\hline
        $\bx_{*,k}$  & $k$-th test example\\\hline
        $f(\bx, \btheta_*)$  & Model \\\hline
        $Y_*$  & Label\\\hline
        $H_t = (X_\ell, Y_\ell)_{\ell \in [t - 1]}$  &  History of $t - 1$ previously labeled examples\\\hline
        $p(\cdot \mid \bx, H_t)$ & Distribution of the label of example $\bx$ conditioned on $H_t$\\\hline
        $\btheta_0$  & Prior mean of the unknown model parameter $\btheta_*$\\\hline
        $\bSigma_0$  & Prior mean of the unknown model parameter $\btheta_*$\\\hline
        $\wSigma_t
  = \left(\bSigma_0^{-1} + \sigma^{-2} \sum_{\ell = 1}^{t - 1} X_\ell X_\ell\T\right)^{-1}$ & Posterior covariance\\\hline
        $\cL_t$ & Set of labeled examples\\\hline
        $\cU_t$  & Set of unlabeled examples\\\hline
        $\wtheta_{t, i, j}$ & Posterior mean\\\hline
    \end{tabular}
    \vspace{1em}
    \caption{Table of Notations}
    \label{tab:my_label}
\end{table}


\end{document}